\documentclass[framed]{customarticle}

\renewcommand{\approx}{{\operatorname{app}}}
\newcommand{\tight}[2]{{#2}}

\title{Quadratic Objective Perturbation: \\ Curvature-Based Differential Privacy}
\author{
Daniel Cortild \orcidCortild (\texttt{daniel.cortild@maths.ox.ac.uk}) \\
Coralia Cartis \orcidCartis (\texttt{coralia.cartis@maths.ox.ac.uk}) \\ \\
Mathematical Institute, University of Oxford, Oxford, United Kingdom}
\date{Last Compiled: \today}

\renewcommand{\red}[1]{{#1}}

\begin{document}

\maketitle

\begin{abstract}
    \noindent
    Objective perturbation is a standard mechanism in differentially private empirical risk minimization. In particular, Linear Objective Perturbation (LOP) enforces privacy by adding a random linear term, while strong convexity and stability are ensured by an additional deterministic quadratic term. However, this approach requires the strong assumption of bounded gradients of the loss function, which excludes many modern machine learning models. In this work, we introduce Quadratic Objective Perturbation (QOP), which perturbs the objective with a random quadratic form. This perturbation induces strong convexity and enforces stability of the problem through curvature, thereby enabling privacy and allowing sensitivity to be controlled through spectral properties of the perturbation rather than assumptions on the gradients. As a result, we obtain $(\varepsilon, \delta)$-differential privacy under weaker \red{gradient} assumptions. Furthermore, we extend the analysis to account for approximate solutions, showing that privacy guarantees are preserved under inexact solves. Additionally, we derive utility guarantees in terms of empirical excess risk, and provide a theoretical and numerical comparison to LOP, highlighting the advantages of curvature-based perturbations. Finally, we discuss algorithmic aspects and show that the resulting problems can be solved efficiently using modern splitting schemes.
    
    \vspace{1em}\noindent\textbf{Keywords.} Differential Privacy, Objective Perturbation, Random Matrix Theory, Three-Operator Splitting.
\end{abstract}
\section{Introduction}

Modern machine learning is largely driven by Empirical Risk Minimization (ERM) problems, where a model is obtained by minimizing a sum of losses over a dataset. Formally, given a dataspace $\mathcal X\subset \R^p\times \R$ and a dataset $\D=\{z_i\}_{1\le i\le n}\in \X^n$ of $n$ individuals, we consider the problem
\begin{equation}\label{eq:p}\tag{P}
    \theta_\exact\red{\in} \argmin_{\theta\in C}\left\{\mathcal J(\theta; \D)=\sum_{i=1}^n \ell (\theta; z_i)+r(\theta)\right\},
\end{equation}
\red{assuming such a point $\theta_\exact$ exists, }where $\ell\colon \R^d\times \X \to \R$ is a loss function convex in its first argument for any given datapoint $z\in \X$, $r\colon \R^d\to \R$ is a convex regularizer that may depend on $n$ but is otherwise independent of the dataset, and $C\subset \R^d$ is a \red{nonempty} closed convex constraint set. In many applications, the parameter space is naturally constrained, for instance to enforce structural properties such as boundedness or feasibility. From a theoretical perspective, the constraint helps ensure well-posedness in the absence of strong convexity. This formulation captures a wide range of problems, from classical linear models to modern overparametrized systems. In many such settings, especially in large dimensions, the problem exhibits many difficulties; minimizers are in general not unique, curvature may be highly anisotropic, and classical regularity assumptions such as strong convexity or bounded gradients often fail. We refer to \cite{negahban_unified_2012} and the references therein for a more complete overview.

Empirical risk minimization is often performed on sensitive data, which we aim to protect through the framework of \textit{Differential Privacy} (DP) \citep{dwork_differential_2006}. Informally, DP guarantees that an adversary cannot reliably infer whether any particular individual contributed to the dataset. Achieving this guarantee, however, comes at a price in terms of utility, especially in the absence of strong regularity, in which case the minimizers might be highly sensitive to perturbations in the data. Balancing privacy and utility has been a central challenge since the introduction of DP, and our work places itself within this line of research. Existing approaches mostly rely on the strong assumption of bounded gradients of the loss functions, which excludes many models and is unrealistic in modern machine learning. A core contribution is that we do not make such an assumption.

\paragraph{Linear Objective Perturbation.} The closest and most relevant prior work is \textit{linear objective perturbation} \citep{chaudhuri_differentially_2011, kifer_private_2012}, in which the private problem is obtained by adding a random linear perturbation to the model. Specifically, the mechanism considers
\begin{equation}\tag{P-Priv-Lin}
    \theta_\priv^\lin=\argmin_{\theta\in C}\left\{\mathcal J_\priv^\lin(\theta; \D)=\J(\theta; \D)+\frac{\Delta}{2}\|\theta\|^2+a^T\theta\right\},
\end{equation}
where $a$ is a random vector enforcing privacy, and $\Delta>0$ is a deterministic scalar ensuring a unique solution. As linear perturbations might significantly shift the minimizer, the deterministic quadratic perturbation is required to achieve stability of the problem. Furthermore, this approach requires the strong regularity assumption of bounded gradients of the loss functions to achieve privacy, to ensure limited individual contributions of each datapoint. This is because the random perturbation does not control the curvature, and hence they must assume bounded variability instead. 
This suggests that a random linear perturbation might not \red{always} be the right choice.
Our model avoids this assumption, by simultaneously handling privacy and stability through the same perturbation. 

The idea of linear objective perturbation was originally introduced in \cite{chaudhuri_differentially_2011}, for pure $\varepsilon$-privacy, and later extended in \cite{kifer_private_2012} for $(\varepsilon, \delta)$-differential privacy. The empirical risk bounds were improved in \cite{jain_dimension_2014}, and an excess population loss was provided in \cite{bassily_private_2019}. Practical aspects were taken into account in \cite{iyengar_practical_2019}, who showed that privacy is preserved even under inexact solves.  More recently, in \cite{redberg_privately_2021}, the notion of \textit{per-instance differential privacy} was studied, and in \cite{redberg_improving_2023}, a new analysis through Rényi differential privacy was provided, showing improved privacy guarantees. 

The strong assumptions required in the above works suggest that linear perturbations are insufficient in our setting, motivating the need for an alternative perturbation mechanism.

\paragraph{Quadratic Objective Perturbation.} To enforce privacy, we propose the \textit{Quadratic Objective Perturbation Mechanism}. Instead of solving Problem \eqref{eq:p}, we draw a random $d\times d$ matrix $W\sim \mathcal W$ \red{and a random $d$-dimensional vector $\tilde \theta$}, and solve instead the perturbed problem 
\red{\begin{equation}\label{eq:p_priv}\tag{P-Priv}
    \theta_\priv = \argmin_{\theta\in C}\left\{\Jp(\theta; \D)=\mathcal J(\theta; \D)+\frac{\sigma^2}{2}(\theta-\tilde\theta)^TW(\theta-\tilde\theta)\right\}.
\end{equation}
}The key feature is that the same random perturbation simultaneously provides privacy, controls the stability of the problem, and induces strong convexity. Unlike the previously mentioned linear perturbations, which shift the minimizers, the random quadratic perturbation modifies the curvature of the objective, ensuring the sensitivity of the output is controlled through the curvature of the objective rather than through assumptions on the gradients. As a result, we can guarantee privacy without requiring bounded gradients of the loss functions.

In practice, the perturbed problem is not solved exactly, and we instead allow for approximate solutions. 

We complement our theoretical findings with numerical experiments, showing that QOP exhibits a fundamentally different scaling behavior compared to LOP. In particular, while LOP's risk increases with the constraint diameter, QOP remains stable, highlighting the advantage of curvature-based control over gradient-based control.

\paragraph{Related Works.} Besides objective perturbation, other ideas have been developed to ensure private ERM. We highlight output perturbation, which adds noise to the final solution \citep{dwork_differential_2006,chaudhuri_differentially_2011,zhang_efficient_2017,lowy_output_2024}, gradient perturbation, which injects noise throughout the algorithm \citep{bassily_private_2014,song_stochastic_2013,abadi_deep_2016,bassily_private_2019}, functional perturbation, which perturbs a representation of the objective \citep{zhang_functional_2012}, input perturbation, which perturbs the data prior to training \citep{duchi_local_2013,fukuchi_differentially_2017}, and sampling-based methods, such as the exponential mechanism, which selects outputs at random proportionally to the exponential utility score \citep{mcsherry_mechanism_2007}.

\paragraph{Contributions.} Our contributions may be summarized as follows:
\begin{enumerate}
    \item We introduce a new objective perturbation mechanism based on random quadratic perturbations, eliminating the need for bounded gradient assumptions required by prior work.
    \item We develop a novel privacy analysis showing that the quadratic perturbation controls the sensitivity of the output through curvature, yielding $(\varepsilon, \delta)$-privacy under weaker \red{gradient} assumptions than prior work. 
    \item We extend the method to approximate solves while preserving meaningful privacy guarantees.
    \item We provide utility guarantees by deriving bounds on the empirical excess risk.
\end{enumerate}

\paragraph{Structure.} The paper is organized as follows. Section \ref{sec:prelims} lists the assumptions and preliminary results. Section \ref{sec:results} introduces the privacy analysis and the utility analysis, along with proof sketches. Section \ref{sec:algorithm} discusses algorithmic considerations, and, finally, Section \ref{sec:LOP} provides both a theoretical and an empirical comparison to linear objective perturbation. All technical results are gathered in the appendix.
\section{Assumptions and Preliminaries}\label{sec:prelims}

In Section \ref{ssec:problem} we formally state the problem assumptions. Before providing a privacy analysis, we define the notion of privacy and give some introductory properties in Section \ref{ssec:dp}. Our analyses, both of privacy and of utility, heavily rely on spectral properties of the random perturbation to control the optimization problem. Preliminaries from random matrix theory are given in Section \ref{ssec:rmt}.

\subsection{Problem Assumptions}\label{ssec:problem}

Our goal is to solve the Problem \eqref{eq:p_priv}. We make the following assumptions on the problem.
\begin{assumption}[Problem Assumptions]\label{ass:problem}
    We assume that
    \begin{itemize}
        \item for each datapoint $z\in \X$, the loss function $\ell(\cdot; z)\colon \R^d \to \R$ is convex, twice \red{continuously} differentiable, is $L$-smooth (i.e. has $L$-Lipschitz continuous gradients), and has a Hessian with rank at most $\rho\le d$;
        \item \red{there exists a publicly known data-independent point $\theta_c\in \R^d$ and scalar $\Delta_c\ge 0$ such that the gradients of the loss functions have bounded diameter at $\theta_c$, namely $\sup_{z, z'\in \X}\|\nabla \ell(\theta_c; z)-\nabla \ell(\theta_c; z')\|\le \Delta_c$;}
        \item the constraint set $C\subset \R^d$ is a closed and convex set; and
        \item the regularizer $r\colon \R^d\to \R$ is convex.
    \end{itemize}
\end{assumption}

\red{The bounded-diameter condition of the gradients at $\theta_c$ is satisfied in several simple settings:
\begin{enumerate}
    \item Suppose that the loss functions have uniformly bounded gradients, namely that $\sup_{\theta\in \R^d, z\in \X}\|\nabla \ell(\theta; z)\|\le \bar G$. Then, for any choice of $\theta_c\in \R^d$, the assumption holds with $\Delta_c=2\bar G$. Consequently, our assumption is weaker than the uniformly bounded gradients assumption typically imposed in analyses of LOP.
    \item Suppose that the interpolation condition is satisfied, namely that there exists $\theta_*\in \cap_{z\in \X}\argmin_{\theta}\ell(\theta; z)$. Since $\ell$ is convex and smooth, this means $\nabla \ell(\theta_*; z)=0$ for all $z\in\X$, and hence the assumption holds with $\theta_c=\theta_*$ and $\Delta_c=0$. If only an approximate $\tilde \theta_*$ is available, satisfying $\|\tilde\theta_*-\theta_*\|\le \eta$, then, by $L$-smoothness of $\ell(\cdot, z)$,
    \[
        \|\nabla \ell(\tilde\theta_*;z)-\nabla \ell(\tilde\theta_*;z')\|\le 2L\eta.
    \]
    As such, the assumption holds with $\theta_c=\tilde \theta_*$ and $\Delta_c=2L\eta$.
    \item Finally, suppose the dataspace $\X$ is compact and that the map $z\mapsto \nabla \ell(\theta_c; z)$ is continuous for some $\theta_c\in \R^d$. Its image is then compact, and therefore bounded, so that $\Delta_c<+\infty$. This condition may hold even when the gradients are unbounded as functions of $\theta$.
\end{enumerate}
These examples show that our condition requiring bounded diameter of the gradients at $\theta_c$ is strictly weaker than requiring uniformly bounded gradients.
}

\paragraph{Approximate Solutions.} We now formalize the approximation error arising from solving the perturbed problem inexactly.
In fact, we do not claim to solve Problem \eqref{eq:p_priv} exactly, but merely that we can solve it up to $\tau$-accuracy. Specifically, we mean that we identify a point $\theta_\approx\in C$ that satisfies 
\red{\begin{equation}\label{eq:approximate_def}
    \Jp(\theta_\approx; \D)-\Jp(\theta_\priv; \D)=\Jp(\theta_\approx; \D)-\min_{\theta\in C}\Jp(\theta; \D)\le \tau.
\end{equation}
}While this criterion is not a practical stopping rule, it is sufficient for our analysis. We introduce a computable surrogate implying it in Appendix \ref{ssec:stopping}.

\paragraph{Notation.} We denote by $\S^d$ and $\S_+^d$ the set of symmetric and symmetric positive semidefinite $d\times d$ matrices. The various optimization problems and their solutions are summarized in Appendix \ref{sec:notation}.

\subsection{Differential Privacy}\label{ssec:dp}

We use the notion of $(\varepsilon, \delta)$-differential privacy, introduced in \cite{dwork_differential_2006}. Intuitively, a randomized mechanism provides DP if it produces similar outputs on similar datasets. The following definitions formalize the notion of \textit{similar datasets} and of \textit{differential privacy}.

\begin{definition}[Replacement Adjacency]\label{def:adjacency}
    We say two datasets $\D$ and $\D'$ are \textit{replacement adjacent}, denoted by $\D\sim \D'$, if they have the same number of datapoints and differ in at most one of them.
\end{definition}

\begin{definition}[Differential Privacy]
    We say a mechanism $\mathcal M$ provides $(\epsilon, \delta)$-\textit{differential privacy} if, for any two datasets $\D\sim \D'$ and any measurable set $S\subset \ran(\mathcal M)$, it holds that
    \[
        \P(\mathcal M(\D)\in S)\le e^\varepsilon\cdot \P(\mathcal M(\D')\in S)+\delta,
    \]
    where $\mathcal M(\D)$ and $\mathcal M(\D')$ are the outputs of $\mathcal M$ under the data $\D$ and $\D'$.
\end{definition}

Much work on differential privacy has been carried out since its introduction, and many results about enforcing privacy have since been introduced. We present some useful results that we will make use of in our main proof in Appendix \ref{sec:dp_lemmas}. For an overview of these results and a more thorough introduction to differential privacy we refer to \cite{dwork_algorithmic_2014} and the references therein.

\subsection{Random Matrix Theory}\label{ssec:rmt}

In order to formally state Problem \eqref{eq:p_priv}, we must first draw a random matrix $W\sim \mathcal W$ from some distribution $\mathcal W$, which enforces privacy and controls the stability of the problem. We make some assumptions on $\mathcal W$, and show in Lemma \ref{lem:satisfies} an example that satisfies it.

\begin{assumption}[Random Matrix Assumption]\label{ass:distribution}
    Let $W$ be distributed according to a probability distribution supported on $\S_+^d$ with density $q\colon \S^d_+\to \R_+$. We assume that \red{$W$ is almost surely invertible, and}, given $\delta_{1}$ and $\delta_{3}$, there exist $\alpha$, $\alpha_1$, $\mu$, \red{and $\nu$} such that
    \[
        \P(\lambda_{\min}(W)\ge \alpha)\ge 1-\delta_{3}, \quad 
         \frac{\P(\alpha\le \lambda_{\min}(W)\le \alpha+\alpha_1)}{\P(\lambda_{\min}(W)\ge \alpha)}\le \delta_{1}, \quad \Exp{W}\preceq \mu I, \red{\quad \text{and}\quad \Exp{W^{-1}}\preceq \nu I}.
    \]
    Moreover, we suppose that there exists a monotone $f\colon \mathbb Z_+\to \R$ such that, provided $\lambda_{\min}(W)\ge \alpha$,
    \[
        \frac{q(W+U)}{q(W)}\le e^{\|U\|_{\op}\cdot f(\rank(U))}\quad \text{for all $W\in \S_+^d$ and $U\in \S^d$ such that $W+U\in \S_+^d$}.
    \]
\end{assumption}

\red{Assumption \ref{ass:distribution} collects the spectral and distributional properties of $W$ needed by our analysis. Informally, it serves the following purposes:
\begin{itemize}
    \item The first two conditions control high-probability spectral events. The bounds justify conditioning $W$ on well-conditioned spectral events in our privacy analysis.
    \item The following two conditions are matrix moment bounds, required for our utility analysis to control the expected cost of the random perturbation. 
    \item Finally, the density-ratio condition controls the change in law of $W$ under low-rank translations, which is the transformation induced by passing between adjacent datasets.
\end{itemize}
}

\textbf{Wishart Ensemble.} We denote by $\Wishart(d, m)$ a $d\times d$ Wishart matrix with hidden dimension $m$, namely that, for $G\in \R^{d\times m}$ with entries $G_{i, j}\overset{\text{i.i.d.}}{\sim} \mathcal N(0, 1)$, $W=GG^T\sim \Wishart(d, m)$. 

The following lemma is a simplified version of Lemma \ref{lem:WishartSatisfiesAss}, which provides the exact values of all the required constants. For technical reasons the full lemma is postponed to the appendix. 

\begin{lemma}[Wishart Distribution Satisfies Random Matrix Assumption]\label{lem:satisfies}
    Let $W\sim \Wishart(d, m)$ be a $d\times d$ Wishart matrix with hidden dimension $m>\red{d+1}$. Then the Random Matrix Assumption \ref{ass:distribution} is satisfied for suitable values of $\alpha$, $\alpha_1$, $f$, $\mu$ \red{and $\nu$} given in Lemma \ref{lem:WishartSatisfiesAss}.
\end{lemma}
\section{Quadratic Objective Perturbation: Privacy and Utility Guarantees}\label{sec:results}

Our approach to privacy relies on perturbing the objective through a random quadratic form, which modifies the curvature of the problem. This allows us to control the sensitivity of the output through spectral properties of the perturbation, rather than through bounded gradients. We now formulate our mechanism, and provide privacy (Section \ref{ssec:privacy}) and utility (Section \ref{ssec:utility}) guarantees. 

\begin{algorithm}[H]
\floatname{algorithm}{Mechanism}
\renewcommand{\thealgorithm}{QOP}
\caption{(Quadratic Objective Perturbation Mechanism)}\label{alg:QOP}
\begin{algorithmic}
\Require \red{reference point $\theta_c\in \R^d$, gradient-diameter bound $\Delta_c\ge 0$,} target privacy parameters $(\varepsilon, \delta)$, solver accuracies $\tau$, and noise parameters \red{$\sigma_0^2$,} $\sigma^2, \tilde \sigma^2$. 
\State Draw $W\sim \mathcal W$\red{ and $a\sim \mathcal N(0, \sigma_0^2 I)$} independently.
\State \red{Define $\tilde \theta$ such that $W(\tilde\theta-\theta_c)=-\sigma^{-2}a$.}
\State Compute a $\tau$-accurate solution $\theta_\approx$ to Problem \eqref{eq:p_priv} satisfying Equation \eqref{eq:approximate_def}.
\State \red{Draw $b\sim \mathcal N(0, \tilde\sigma^2 I)$.}
\State \textbf{Release} $\theta_{\final}=\theta_{\approx}+b$.
\end{algorithmic}
\end{algorithm}
\red{The mechanism samples three sources of randomness, each corresponding to a separate component of the privacy argument. Specifically:
\begin{itemize}
    \item The random matrix $W$ determines the random curvature added to the problem, and is the central element of our privacy accounting. As a byproduct, it also convexifies and stabilizes the problem. 
    \item The random vector $a$ randomizes the center of the quadratic perturbation. Importantly, the losses may be re-centered so that they all share the common minimizer $\theta_c$, which is required by our analysis.
    \item Finally, the random vector $b$ is added to the approximate solution before release. This additional Gaussian term protects against data dependence of the optimization error, and is unnecessary when the problem is solved exactly. This component resembles classical output perturbation \cite{dwork_differential_2006,chaudhuri_differentially_2011,zhang_efficient_2017,lowy_output_2024}. However, unlike output perturbation, where the Gaussian noise is responsible for the entire privacy guarantee, here it is calibrated solely to account for the optimization error.
\end{itemize}
Though the definition of $\tilde\theta$ involves solving a linear system, note that 
\[
    \frac{\sigma^2}{2}(\theta-\tilde\theta)^TW(\theta-\tilde\theta)=a^T(\theta-\theta_c)+\frac{\sigma^2}{2}(\theta-\theta_c)^TW(\theta-\theta_c)+\frac{1}{2\sigma^2}a^TW^{-1}a,
\]
such that Problem \eqref{eq:p_priv} may equivalently be rewritten as 
\begin{equation*}
        \theta_\priv = \argmin_{\theta\in C}\left\{\J(\theta; \D)+a^T(\theta-\theta_c)+\frac{\sigma^2}{2}(\theta-\theta_c)^TW(\theta-\theta_c)\right\},
    \end{equation*}
As such, Problem \eqref{eq:p_priv} may be written with a linear perturbation and a quadratic perturbation with a deterministic center, rather than a random quadratic with a random center as originally presented. This notation makes it clear that solving a system involving $W$ is not necessary for implementation purposes. 
}

We now delve into the privacy and utility properties of this mechanism.

\subsection{Privacy Analysis}\label{ssec:privacy}

We now state our main result, the $(\varepsilon, \delta)$-privacy guarantee of our mechanism. Intuitively, privacy is preserved if the noise level associated with the curvature addition and with the approximate release are sufficiently high. The following theorem provides quantitative parameters to ensure this.

\begin{theorem}[Differential Privacy of Mechanism \ref{alg:QOP}]\label{thm:main}
    Assume that the Problem Assumption \ref{ass:problem} and the Random Matrix Assumption \ref{ass:distribution} hold. \red{Let $\varepsilon\in (0, 1)$ and $\delta>0$.} Then Mechanism \ref{alg:QOP} is $(\varepsilon, \delta)$-differentially private, provided 
    \[
        \red{\sigma_0\ge \frac{\Delta_c}{\varepsilon_0}\sqrt{2\log \frac{1.25}{\delta_0}}, \quad}\sigma^2\ge \max\left(\frac{2L}{\varepsilon_1}\left(f(2)+\frac{2(2\rho+2)}{\alpha}\right), \frac{2L}{\alpha_1}\right) \quad \text{and}\quad \red{\tilde \sigma\ge\frac{2}{\varepsilon_2}\sqrt{\frac{2\tau}{\alpha\sigma^2}}\sqrt{2\ln(1.25/\delta_2)}},
    \]
    where $\varepsilon=\red{\varepsilon_0+}\varepsilon_1+\varepsilon_2$, $\delta=\red{\delta_0+}\delta_1+\delta_2+\delta_3$, and $\alpha, \alpha_1$ and $f(2)$ are given in the Random Matrix Assumption \ref{ass:distribution}.
\end{theorem}
{\begin{proof}[Proof Idea]
    We provide a proof idea here, and the full details are available in Appendix \ref{sec::proof_dp}.

    \red{First, define $Y=\sum_{i=1}^n\nabla \ell(\theta_c; z_i)+a$, such that, given that $Y=y$, we may rewrite the objective as
    \begin{equation*}
        \theta_\priv = \argmin_{\theta\in C}\left\{\sum_{i=1}^n\bar\ell(\theta; z_i)+\bar r(\theta)+\frac{\sigma^2}{2}(\theta-\theta_c)^TW(\theta-\theta_c)\right\},
    \end{equation*}
    where $\bar \ell(\theta; z)=\ell(\theta; z)-\nabla \ell(\theta_c; z)^T(\theta-\theta_c)$ and $\bar r(\theta)=r(\theta)+y^T(\theta-\theta_c)$. Note that $\bar\ell$ and $\bar r$ still satisfy Problem Assumption \ref{ass:problem}, with the additional key fact that $\theta_c\in \argmin \bar\ell(\cdot; z)$ for all $z\in \X$. 
    }

    \red{Although the perturbed objective is almost surely strongly convex, we condition the mechanism on $W$ having uniformly bounded below eigenvalues by $\alpha$, making the perturbed objective $\alpha\sigma^2$-strongly convex.} As this occurs with high probability, the Conditional Mechanism Lemma \ref{lem:conditional} guarantees that privacy is preserved.
    
    In order to prove that the release of $\theta_\final$ is private, we shall show that the joint release of \red{$(Y, \theta_\priv, \theta_\final-\theta_\priv)$} is private, by applying the Composition of Mechanisms Lemma \ref{lem:composition}. The latter two releases \red{determine $\theta_\final$}, and hence can only worsen privacy. This idea is inspired by \cite{iyengar_practical_2019}. As such, the proof is split into \red{three} parts, proving the privacy of the release of\red{ $Y$}, of $\theta_\priv$ and of $\theta_\final-\theta_\priv$;
    \begin{enumerate}
        \red{\item The release of $Y$ is easily shown private by the Gaussian Mechanism \ref{lem:gauss}. Conditional to the release of $Y$, the above reformulation allows us to assume all loss functions are minimized at the common minimizer $\theta_c$. This is described in Lemma \ref{lem:releaseY}.}
        \item To show that the release of $\theta_\priv$ is private, we first show it is private under a simplification assumption, namely that the regularizer $r$ is twice differentiable and that the constraint set $C=\R^d$. This is the main novelty of the proof, presented in Lemma \ref{lem:part1_simple}, and largely extends the ideas in \cite{chaudhuri_differentially_2011} and \cite{kifer_private_2012}. We then apply the Successive Approximation Lemma \ref{lem:sucessive} to remove the simplification assumptions one-by-one, as done in \cite{kifer_private_2012}. This is described in Lemma \ref{lem:part1}.
        \item To show that the release of $\theta_\final-\theta_\priv$ is private, it is sufficient to note that $\theta_\final-\theta_\priv=\theta_\approx-\theta_\priv+b$, where $b$ is Gaussian noise. In this case, by the Gaussian Mechanism \ref{lem:gauss}, the process is private if $\theta_\approx-\theta_\priv$ is bounded uniformly in the noise realization, which it is. This is described in Lemma \ref{lem:part2}. \qedhere
    \end{enumerate}
\end{proof}}
Theorem \ref{thm:main} shows that privacy is ensured by combining \red{three} mechanisms: \red{the random center, which re-centers the loss functions,} the random quadratic perturbation, which controls the sensitivity of the exact minimizer through curvature, and the Gaussian mechanism, which accounts for approximation errors.

\subsection{Utility Analysis}\label{ssec:utility}

We now analyze the empirical excess risk of the proposed mechanism. The bound reflects the trade-off between privacy and solver accuracy.

Our utility analysis requires an additional assumption on the regularizer, given by the following. 
\begin{assumption}[Bounded Subgradient of Regularizer]\label{ass:problem_stronger}
    We assume that the regularizer $r$ has subgradients bounded by $G\ge 0$.
\end{assumption}
We note that one of our core contributions is that we do not require bounded gradients of the loss functions. The above assumption is not in contradiction with this, as it is an assumption on the regularizer. In fact, the assumption is quite standard for regularizers, and holds for the classical regularizers $r=\|\cdot\|_1$, $r=\|\cdot\|_2$ or $r=\|\cdot\|_*$.

We now state our utility result, whose proof is postponed to Appendix \ref{sec:proof_utility}.

\begin{theorem}[Empirical Excess Risk Bound]\label{thm:utility}
    Assume that the Problem Assumption \ref{ass:problem}, the Bounded Regularizer Subgradients Assumption \ref{ass:problem_stronger}, and the Random Matrix Assumption \ref{ass:distribution} hold. Moreover, select \red{$\sigma_0^2$}, $\sigma^2$ and $\tilde \sigma^2$ as in Theorem \ref{thm:main}, to ensure $(\varepsilon, \delta)$-differential privacy of Mechanism \ref{alg:QOP}. Then it holds that, for $\theta_\final$, the output of Mechanism \ref{alg:QOP} and $\theta_\exact$, the solution of the Unperturbed Problem \eqref{eq:p}, 
    \red{\[
        \Exp{\J (\theta_\final; \D)-\J(\theta_\exact; \D)}\le \frac{ndL\tilde\sigma^2}{2} +G\sqrt{d}\tilde\sigma+\tau+\frac{\mu\sigma^2}{2}\cdot \|\theta_\exact-\theta_c\|^2+\frac{\sigma_0^2d\nu }{2\sigma^2}.
    \]}
\end{theorem}

In case the subproblem may be solved exactly, the above substantially simplifies, as many of the user-chosen parameters may be set to $0$.
Moreover, the Bounded Regularizer Subgradients Assumption \ref{ass:problem_stronger} is not needed in that case. 

\begin{corollary}[Empirical Excess Risk Bound for Exact Solve]\label{coro:utility_exact}
    Assume that the Problem Assumption \ref{ass:problem} and the Random Matrix Assumption \ref{ass:distribution} hold, and that the solver is exact, namely that \red{$\tau=0$}. Moreover, select \red{$\sigma_0^2$ and }$\sigma^2$ as in Theorem \ref{thm:main} and set $(\varepsilon_2, \delta_2)=0$, and $\tilde \sigma^2=0$, to ensure $(\varepsilon, \delta)$-differential privacy of Mechanism \ref{alg:QOP}. Then it holds that, for $\theta_\final$, the output of Mechanism \ref{alg:QOP}, and $\theta_\exact$, the solution of the Unperturbed Problem \eqref{eq:p}, 
    \red{\[
        \Exp{\J (\theta_\final; \D)-\J(\theta_\exact; \D)}\le \frac{\mu\sigma^2}{2}\cdot \|\theta_\exact-\theta_c\|^2+\frac{\sigma_0^2d\nu }{2\sigma^2}.
    \]}
\end{corollary}

We note that the bounds in Theorem \ref{thm:utility} and in Corollary \ref{coro:utility_exact} depend on several user-chosen parameters, such as the privacy split $\varepsilon=\red{\varepsilon_0}+\varepsilon_1+\varepsilon_2$ and $\delta=\red{\delta_0+}\delta_1+\delta_2+\delta_3$, as well as on the choice of distribution $\mathcal W$. In particular, when $\mathcal W$ is taken to be a Wishart distribution, as laid out in Lemma \ref{lem:WishartSatisfiesAss}, the hidden dimension $m$ introduces an additional degree of freedom. 

In order to illustrate the obtained utility bounds, we numerically optimize the right-hand side of Theorem \ref{thm:utility} over all free parameters. The resulting optimal values are displayed in Figure \ref{fig:utility} for a range of $(\varepsilon, \delta)$. In these experiments, we fix $L=1$, $d=12$, $n=10$, \red{$\|\theta_\exact-\theta_c\|^2\le 1$, $\tau=1$}, $\rho=1$, $G=1$,\red{ and varying values of $\Delta_c$}. \red{An initial experiment also included varying values of $\tau$, but it was observed that $\tau$ has much smaller influence than $\Delta_c$ on the results.} We solve using the Nelder-Mead algorithm implemented in SciPy \citep{gao_implementing_2012}.

Figure \ref{fig:utility} highlights the expected trade-offs: as the failure probability $\delta$ or the privacy budget $\varepsilon$ increase, the optimal expected loss decreases. \red{Moreover, the magnitude of the optimal expected loss increases as $\Delta_c$ increases, which is to be expected as $\Delta_c$ measures heterogeneity of the gradients at $\theta_c$.}

\begin{figure}[H]
    \centering
    \includegraphics[width=0.8\linewidth]{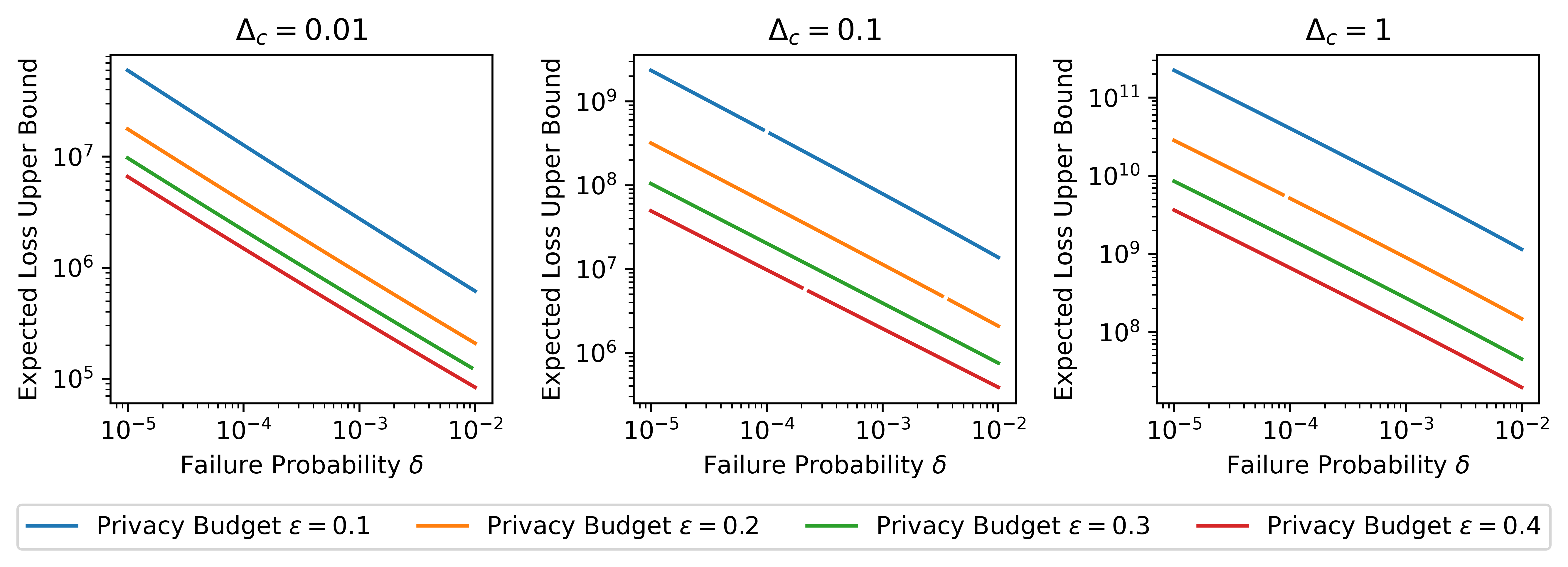}
    \caption{\red{Optimal upper bound on expected loss obtained by minimizing the right-hand side of Theorem \ref{thm:utility} over all free parameters.}}
    \label{fig:utility}
\end{figure}
\section{Algorithmic Considerations}\label{sec:algorithm}

A key observation is that the perturbed problems arising from both LOP and QOP, namely Problems \eqref{eq:p_priv_lin} and \eqref{eq:p_priv}, admit a common composite structure. For a fixed realization of the perturbation, they can be written as 
\[
    \min_{\theta\in \R^d}\left\{\sum_{i=1}^n\tilde \ell(\theta; z_i)+r(\theta)+\iota_C(\theta)\right\},
\]
where $\tilde \ell$ is a smooth convex function, $r$ is a convex (possibly nonsmooth) regularizer, and $\iota_C$ is the indicator function of the convex constraint set $C$. 

This naturally fits within the framework of \textit{Three-Operator Splitting} methods \citep{davis_threeoperator_2017a}, which are designed to handle sums of one smooth convex term and two nonsmooth convex terms. In particular, this allows us to separate the smooth structure of $\tilde \ell$ from the nonsmooth one of $r$ and $\iota_C$. 

In this work, we employ a stochastic variant of the method, leveraging the finite-sum structure of the smooth component. This was introduced in \cite{yurtsever_stochastic_2016,yurtsever_three_2021} under strong assumptions, and refined in \cite{cortild_stochastic_2026}. Specifically, under suitable assumptions, the scheme converges to a solution of the perturbed problems, and is appropriate for the approximate solves required by our analysis. We provide additional details and the full algorithm in Appendix \ref{sec:ass_compl}.
\section{Comparison to Linear Objective Perturbation}\label{sec:LOP}

We now compare our Quadratic Objective Perturbation (QOP) mechanism to the classical Linear Objective Perturbation (LOP) mechanism \citep{chaudhuri_differentially_2011,kifer_private_2012}. While both approaches perturb the objective to ensure privacy, they differ fundamentally in how they control the stability of the problem. This comparison is performed on two fronts, one being theoretical, in Section \ref{ssec:LOP_theory}, the other being numerical, in Section \ref{ssec:LOP_numerical}.

\subsection{Theoretical Comparison}\label{ssec:LOP_theory}

The LOP mechanism also solves a perturbed version of Problem \eqref{eq:p}, however with a random linear perturbation and a deterministic quadratic term. Specifically, it aims at solving 
\begin{equation}\label{eq:p_priv_lin}\tag{P-Priv-Lin}
    \theta_\priv^\lin=\argmin_{\theta\in C}\left\{\mathcal J_\priv^\lin(\theta; \D)=\sum_{i=1}^n\ell(\theta; z_i)+r(\theta)+\frac{\Delta}{2}\|\theta\|^2+a^T\theta\right\}.
\end{equation}
While follow-up works have considered inexact solvers \citep{iyengar_practical_2019}, they do not account for the presence of a regularizer. As such, we shall only compare the exact solve setting.

\begin{algorithm}[H]
\floatname{algorithm}{Mechanism}
\renewcommand{\thealgorithm}{LOP}
\caption{(Linear Objective Perturbation Mechanism)}\label{alg:LOP}
\begin{algorithmic}
\Require \\
\begin{itemize}
    \item target privacy parameters $(\varepsilon, \delta)$,
    \item appropriate parameters $\sigma^2$ and $\Delta$.
\end{itemize}
\State Draw $a\sim \mathcal N(0, \sigma^2 I)$. 
\State Compute the solution $\theta_\priv$ to Problem \eqref{eq:p_priv_lin}.
\State \textbf{Release} $\theta_{\priv}$.
\end{algorithmic}
\end{algorithm}

The LOP mechanism requires a different set of assumptions than the QOP mechanism. They are, however, comparable, and we thus introduce them as an extension of the Problem Assumption \ref{ass:problem}. 

\begin{assumption}[Modified Problem Assumption]\label{ass:problem_mod}
    We assume the Problem Assumption \ref{ass:problem}, with the following modifications:
    \begin{itemize}
        \item we assume the Hessian of the loss function $\nabla^2 \ell(\cdot, z)$ has rank at most $\rho=1$ for all $z\in \mathcal X$,
        \item we assume that all $\ell(\cdot, z)$ have uniformly bounded gradients bounded by $\zeta$. 
    \end{itemize}
\end{assumption}
\red{The latter assumption is an instance of our bounded gradient-diameter assumption, with $\Delta_c=2\zeta$ and any $\theta_c\in \R^d$. For a fair comparison, as the deterministic quadratic in LOP is centered at the origin, we select $\theta_c=0$ for the comparison.}

With these assumptions at hand, it was shown in \cite{kifer_private_2012} that Mechanism \ref{alg:LOP} is $(\varepsilon, \delta)$-differentially private and enjoys risk bounds. 

\begin{theorem}[{\thmcite[Theorems 2 and 4]{kifer_private_2012}}]\label{thm:utility_LOP}
    Assume that the Modified Problem Assumption \ref{ass:problem_mod} holds. Select $\sigma^2$ and $\Delta$ such that 
    \tight{
        $\sigma^2\ge {\zeta^2(8\log\tfrac2\delta+4\varepsilon)}/{\varepsilon^2}$ and $\Delta\ge {2L}/{\varepsilon}$.
    }{
    \[
        \sigma^2\ge \frac{\zeta^2(8\log\tfrac2\delta+4\varepsilon)}{\varepsilon^2}\quad \text{and}\quad \Delta\ge \frac{2L}{\varepsilon}.
    \]
    }Then Mechanism \ref{alg:LOP} is $(\varepsilon, \delta)$-differentially private. Moreover, for $\theta_\priv^\lin$ the solution of Problem \eqref{eq:p_priv_lin} and $\theta_\exact$ the solution of \tight{}{the Unperturbed} Problem \eqref{eq:p}, \tight{}{it holds that }
    \tight{\[
        \Exp{\J(\theta_\priv^\lin; \D)-\J(\theta_\exact; \D)}\le {2\sigma^2d}/{\Delta}+{\Delta}\|\theta_\exact\|^2/{2}.
    \]}{
    \[
        \Exp{\J(\theta_\priv^\lin; \D)-\J(\theta_\exact; \D)}\le \frac{2\sigma^2d}{\Delta}+\frac{\Delta}{2}\|\theta_\exact\|^2.
    \]
    }
\end{theorem}

In summary, Mechanism \ref{alg:QOP} differs from Mechanism \ref{alg:LOP} in two key aspects: it controls the sensitivity through random curvature, and it integrates privacy, stability and strong convexification directly into the random perturbation rather than requiring an additional perturbation for it. This allows Mechanism \ref{alg:QOP} to apply to a broader class of problems and provide a more natural framework for balancing privacy and utility. We summarize the main differences in Table \ref{tab:LOPvsQOP}.

\begin{table}[H]
    \centering
    \caption{Comparison between Linear and Quadratic Objective Mechanisms.}
    \begin{tabular}{|l|c|c|} \hline
        \textbf{Property} & \textbf{Mechanism \ref{alg:LOP}} & \textbf{Mechanism \ref{alg:QOP}} \\ \hline
        Random Perturbation & Linear & Quadratic \\
        Stability Mechanism & Deterministic Quadratic & Intrinsic through Curvature \\
        \red{Gradient Assumption} & \red{Uniformly Bounded} & \red{Bounded Diameter at $\theta_c$} \\ \hline
    \end{tabular}
    \label{tab:LOPvsQOP}
\end{table}

\subsection{Numerical Comparison}\label{ssec:LOP_numerical}

We now empirically demonstrate a fundamental scaling difference between Mechanisms \ref{alg:LOP} and \ref{alg:QOP}: in settings with unbounded gradients, LOP's risk scales with the diameter of the constraint set, whereas QOP's risk remains invariant. We compare only the curvature-based perturbations to the linear perturbations, and not the \red{solver-error corrections}. 

We consider the problem of linear least squares with LASSO regression. Specifically, given a dataset $\D=\{z_i=(x_i, y_i)\}\subset [-\xi, \xi]^d\times \red{[-1, 1]}$, the problem is given by 
\[
    \min_{\theta\in C}\left\{\sum_{i=1}^n \frac12(x_i^T\theta-y_i)^2+\omega\|\theta\|_1\right\},
\]
where $\omega$ is a regularization parameter. The regularizer has bounded subgradients by $G=\sqrt{d}\omega$, each loss function is $L$-smooth for $L=d\xi^2$ and has a Hessian of rank at most $\rho=1$, and the loss function has bounded gradients when the domain is bounded, with, for all $z\in \X$,
\[
    \|\nabla \ell(\theta; z)\|\le \zeta=(\diam(C)\xi\sqrt{d} +1)\cdot \xi\sqrt{d}.
\]
\red{For a fair comparison between LOP and QOP, we fix $\theta_c=0$, at which the bounded gradient-diameter assumption is satisfied with
\[
    \Delta_c=\sup_{z, z'\in \X}\|\nabla \ell(\theta_c; z)- \nabla \ell(\theta_c; z')\|=\sup_{(x, y), (x', y')\in \X}\|xy-x'y'\|\le 2\xi\sqrt{d}.
\]}

The resulting Problems \eqref{eq:p_priv_lin} and \eqref{eq:p_priv} admit a composite structure, so we use  the \textit{Stochastic Three-Operator Splitting scheme} mentioned in Section \ref{sec:algorithm} (and
described in Appendix \ref{sec:ass_compl}). The experimental setup is given in Appendix \ref{sec:experimental}, and the resulting empirical risks are plotted in Figure \ref{fig:risk}. Statistical quantities, such as standard deviation and runtime, are reported in Section \ref{ssec:stats}.

\begin{figure}[H]
    \centering
    \includegraphics[width=1\linewidth]{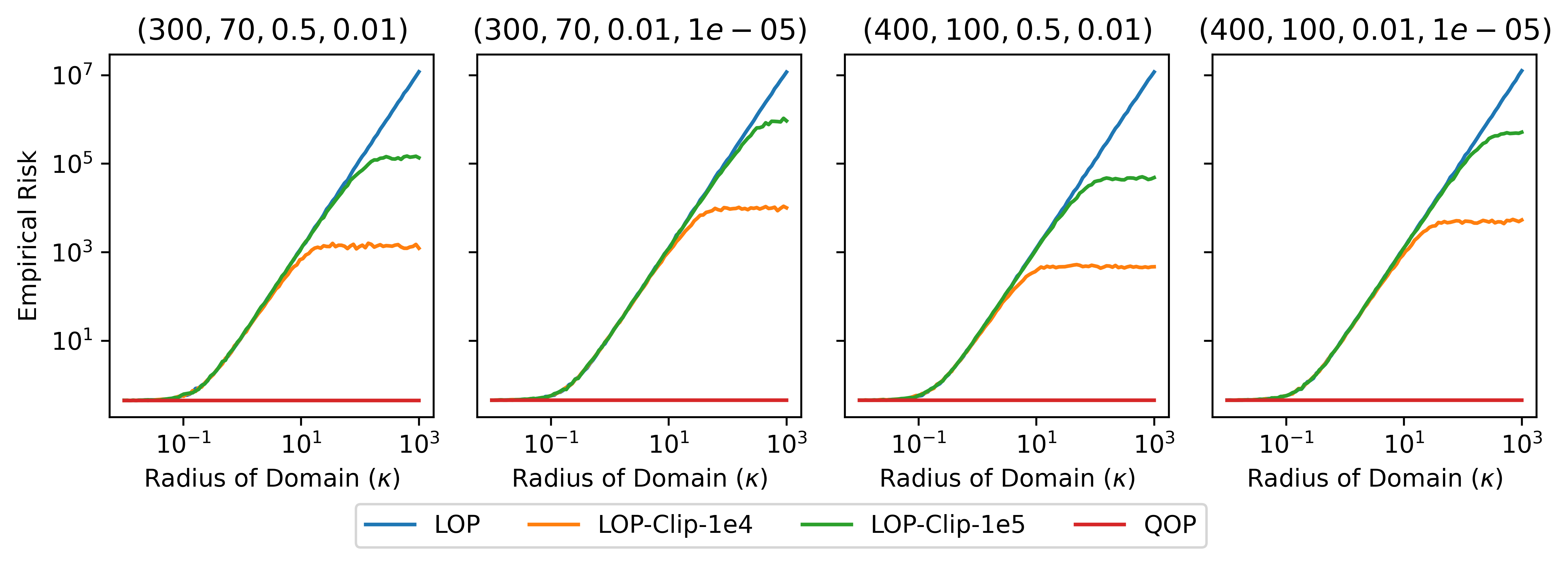}
    \caption{\red{Empirical risks plotted on a logarithmic scale for various values of $(n, d, \varepsilon, \delta)$.}}
    \label{fig:risk}
\end{figure}

We also compare the two aforementioned mechanisms to LOP-Clip, which clips the gradient used in LOP. This mechanism has provable privacy guarantees \citep{redberg_improving_2023}, though dependent on the clipping threshold, which is an additional user-selected parameter.

As predicted by theory, LOP requires noise proportional to the gradient bound, which scales with $\kappa$ (appearing linear on the logarithmic scale), whereas QOP controls sensitivity via curvature \red{and does not directly depend on the constraint diameter, and, empirically, its risk is observed stable with respect to $\kappa$}. Overall, our experiments complement and illustrate the theory.

\section{Conclusion}

In this work, we introduced Quadratic Objective Perturbation (QOP) for differentially private empirical risk minimization. In contrast with Linear Objective Perturbation (LOP), which relies on strong bounded gradient assumptions, our mechanism enforces privacy through curvature while simultaneously inducing stability and strong convexity, yielding $(\varepsilon, \delta)$-differential privacy guarantees under significantly weaker assumptions.

We established privacy and utility guarantees, extended the analysis to approximate solutions, and showed that the resulting problem can be solved efficiently due to its composite structure.

Our theoretical and numerical results demonstrate that QOP and LOP satisfy fundamentally different scaling laws when the gradients are unbounded. In particular, LOP degrades with the constraint diameter, whereas QOP remains stable. These findings suggest that curvature-based perturbations provide a natural alternative to classical objective perturbation methods, particularly in modern machine learning settings where strong regularity assumptions are rarely met or unrealistic.

\textbf{Future Work.} 
Several directions remain open.
\red{A first avenue would be to extend the current utility analysis, including population risk, generalization bounds or high-probability guarantees.}
Another direction is to strengthen the privacy analysis by considering Rényi differential privacy rather than standard differential privacy. \red{Finally, it would be of interest to study the case of pure DP, and how the noise assumptions would be updated accordingly.}

\paragraph{Acknowledgments.} Daniel Cortild acknowledges the support of the Clarendon Funds Scholarships. Coralia Cartis' work was supported by the Hong Kong Innovation and Technology Commission (InnoHK Project CIMDA) and by the EPSRC grant EP/Y028872/1, Mathematical Foundations of Intelligence: An “Erlangen Programme” for AI.

\bibliographystyle{apalike}
\bibliography{references}

\newpage
\appendix

\section{Summary of Problems and Notation}\label{sec:notation}

For convenience, we summarize the different optimization problems and their associated solutions used throughout the paper. Table \ref{tab:explanation} is intended as a reference for the various quantities appearing in the main text and appendices.

\begin{table}[H]
    \centering
    \caption{Summary of problems and their corresponding solutions used throughout the paper.}
    \begin{tabular}{|c|l|l|}\hline
        \textbf{Problem} & \textbf{Solution} & \textbf{Description} \\ \hline
        \multirow{2}{*}{-} 
            & \red{$\theta_c$} & \red{point at which Assumption \ref{ass:problem} holds} \\
            & \red{$\tilde \theta$} & \red{random center of quadratic perturbation} \\ \hline
        $\min_{\theta\in C}\mathcal J(\theta; D)$ 
            & $\theta_\exact$ & exact solution to unperturbed problem \\ \hline
        \multirow{2}{*}{$\min_{\theta\in C}\Jp(\theta; D)$} 
            & $\theta_\priv$ & exact solution to perturbed problem \\
            & $\theta_\approx$ & approximate solution to perturbed problem \\ \hline
    \end{tabular}
    \label{tab:explanation}
\end{table}
\section{Preliminary Results from Differential Privacy}\label{sec:dp_lemmas}

These preliminary lemmas complement the short discussion on differential privacy in Section \ref{ssec:dp}. Most lemmas, unless stated otherwise, are extracted from \cite{dwork_algorithmic_2014}.

\begin{lemma}[Gaussian Mechanism]\label{lem:gauss}
    Let $f\colon \X\to \R^p$ have $\ell_2$-sensitivity $\Delta$, namely 
    \[
        \Delta=\max_{\D\sim \D'\in \X}\|f(\D)-f(\D')\|_2.
    \]
    Define the randomized mechanism $\mathcal M(\D)=f(\D)+b$, where $b\sim \mathcal N(0, \sigma^2I)$, where $\sigma\ge \frac{\Delta}{\varepsilon}\sqrt{2\ln(1.25/\delta)}$ \red{with $\varepsilon\in (0, 1)$ and $\delta>0$}. Then $\mathcal M$ is $(\varepsilon, \delta)$-DP. 
\end{lemma}

\begin{lemma}[Composition of Mechanisms]\label{lem:composition}
    Let $\mathcal M_1, \ldots, \mathcal M_k$ be a sequence of mechanisms, possibly dependent on the outcomes of the previous mechanisms (namely, $\mathcal M_2$ may depend on the data $\D$ and on $\mathcal M_1(\D)$). Consider the composed mechanism $\mathcal M$, given by 
    \[
        \mathcal M(\D)=(\mathcal M_1(\D), \mathcal M_2(\mathcal M_1(\D), \D), \ldots ).
    \]
    If $\mathcal M_i$ is $(\varepsilon_i, \delta_i)$-DP, then $\mathcal M$ is $(\sum \varepsilon_i,\sum \delta_i)$-DP.
\end{lemma}

\begin{lemma}[Conditional Mechanism]\label{lem:conditional}
    Let $\mathcal M$ be a randomized mechanism and let $E$ be a data-independent event. If $E$ occurs with probability at least $1-\tilde \delta$, and $\mathcal M$, conditioned on $E$, is $(\varepsilon, \delta)$-DP, then $\mathcal M$ is overall $(\varepsilon, \delta+\tilde \delta)$-DP.
\end{lemma}

\begin{lemma}[Successive Approximation {(\thmcite[Theorem 1]{kifer_private_2012})}]\label{lem:sucessive}
    Let $\mathcal M$ be a randomized mechanism induced by a deterministic function $\phi$ and a random variable $X$, that is $\mathcal M(\D)=\phi(X; \D)$. Let $\mathcal M_1, \mathcal M_2, \ldots$ be a sequence of randomized mechanisms, which each $\mathcal M_i$ is induced by a deterministic function $\phi^i$ and the random variable $X$, namely $\mathcal M_i(\D)=\phi^i(X; \D)$. If all the mechanisms $\mathcal M_i$ are $(\varepsilon, \delta)$-DP and the sequence $(\phi^i)$ converges pointwise to $\phi$ for all datasets $\D$ and all realizations of $X$, then $\mathcal M$ is also $(\varepsilon, \delta)$-DP.
\end{lemma}
\section{Auxiliary Results}

\begin{lemma}[Density Transportation]\label{lem:transport}
    Let $X\in \Omega\subset \R^a$ be a random variable with density $q\colon \Omega\to \R_+$. Let $T\colon \Omega\to \R^b$ be a locally Lipschitz continuous map with $a\ge b$, and define the random variable $Y=T(X)$. If $\|J_bT(x)\|>0$ with probability $1$ on $\{q>0\}$, then $Y$ has density function $p\colon \R^b\to \R_+$ given by 
    \[
        p(y)=\int_{T^{-1}(y)}\frac{q(x)}{\|J_bT(x)\|}d\H^{a-b}(x),
    \]
where $J_bT(x)$ is the $b$-dimensional Jacobian whose determinant is given by $\|J_bT(x)\|=\sqrt{\det(DT(x)DT(x)^T}$, and $d\H^{a-b}$ is the $a-b$-dimensional Hausdorff measure. 
\end{lemma}
\begin{proof}
    For any test function $\phi\in C_c^\infty(\R^b)\subset L^1(\R^b)$, it holds that, by the Coarea formula \citep{federer_curvature_1959, negro_sample_2024},
    \tight{\begin{align*}
        \Exp{\phi(Y)}
        &=\int_\Omega \phi(T(x))q(x)dx\\
        &=\int_\Omega \frac{\phi(T(x))q(x)}{\|J_bT(x)\|}\|J_bT(x)\|dx\\
        &=\int_{\R^b}\int_{T^{-1}(y)}\frac{\phi(T(x))q(x)}{\|J_bT(x)\|}d\H^{a-b}(x)dy.
    \end{align*}}
    {\[
        \Exp{\phi(Y)}=\int_\Omega \phi(T(x))q(x)dx=\int_\Omega \frac{\phi(T(x))q(x)}{\|J_bT(x)\|}\|J_bT(x)\|dx=\int_{\R^b}\int_{T^{-1}(y)}\frac{\phi(T(x))q(x)}{\|J_bT(x)\|}d\H^{a-b}(x)dy.
    \]}
    We note that, on the fiber $T^{-1}(y)$, $T(x)=y$, such that
    \[
        \Exp{\phi(Y)}=\int_{\R^b}\phi(y)\int_{T^{-1}(y)}\frac{q(x)}{\|J_bT(x)\|}d\H^{a-b}(x)dy.
    \]
    As this holds true for all test functions, the statement holds true.
\end{proof}

\begin{lemma}[Rank-1 Update Formula]\label{lem:rank1update}
    Let $A\in \S^d$ have full-rank and let $E\in \S^d$ have rank at most $1$. Then it holds that
    \[
        \frac{\det(A+E)}{\det(A)}=1+\lambda,
    \]
    where $\lambda$ is the unique \red{possible} non-zero eigenvalue of $A^{-1}E$, \red{with $\lambda=0$ if it has none}.
\end{lemma}
\begin{proof}
    \red{The case $E=0$ is trivial. Otherwise}
    denote by $\lambda_{\min}(E)=\lambda_d(E)\le \cdots\le \lambda_1(E)=\lambda_{\max}(E)$. The result follows from 
    \[
        \frac{\det(A+E)}{\det(A)}=\det(I+A^{-1}E)=\prod_{i=1}^d(1+\lambda_i(A^{-1}E))=1+\lambda,
    \]
    since $\rank(E)\le 1$. 
\end{proof}

\begin{lemma}[Low-Rank Update Bound]\label{lem:lowrankupdate}
    Let $A\in \S^d_+$ have full-rank and absolute eigenvalues lower-bounded by $\alpha>0$, and let $E\in \S^d$ have rank at most $r\le d$ and absolute eigenvalues upper-bounded by $L$. Then it holds that
    \[
        \frac{\det(A+E)}{\det(A)}\le \exp(Lr/\alpha).
    \]
\end{lemma}
\begin{proof}
    Denote by $\lambda_{\min}(E)=\lambda_d(E)\le \cdots\le \lambda_1(E)=\lambda_{\max}(E)$.
    It holds that 
    \[
        \frac{\det(A+E)}{\det(A)}=\det(I+A^{-1}E)=\prod_{i=1}^d(1+\lambda_i(A^{-1}E))\le \prod_{i=1}^r(1+|\lambda_i(A^{-1}E)|),
    \]
    where the final inequality follows by $\rank(A^{-1}E)\le \rank(E)\le r$. By bounding $|\lambda_i(A^{-1}E)|\le {L}/{\alpha}$,
    \[
        \frac{\det(A+E)}{\det(A)}\le \left(1+\frac{L}{\alpha}\right)^r\le \exp(Lr/\alpha),
    \]
    where the latter follows by exponential approximation.
\end{proof}

\begin{lemma}\label{lem:matrixU}
    Given a vector $u\in \R^d$ and a vector $b\in \R^d\backslash\{0\}$, there exists a matrix $U\in \S^d$ such that $\rank(U)\le 2$, $Ub=u$ and $\|U\|_{\op}= {\|u\|}/{\|b\|}$.
\end{lemma}
\begin{proof}
    Define $e=\frac{b}{\|b\|}$, let $\alpha=e^Tu$, and let $w=u-\alpha e$, such that $u=\alpha e+w$ and $e^Tw=0$. Define 
    \[
        P_v=\begin{cases}
            \frac{vv^T}{\|v\|^2} \quad &\text{if $v\neq 0$}, \\
            0 & \text{else},
        \end{cases}
    \]
    and 
    \[
        U=\frac{1}{\|b\|}\left(\alpha(P_e-P_w)+ew^T+we^T\right).    \]
    It is easy to see that $\rank(U)\le 2$ since its column space is contained in $\spa(e, w)$. It is then easy to check using $w^Te=0$ that
    \[
        Ub=\left(\alpha(P_e-P_w)+ew^T+we^T\right)e=\alpha e+w=u,
    \]
    as wanted. It remains to bound the norm of $U$. For this, denote $S=ew^T+we^T$ and note that 
    \tight{\[
        S(P_e-P_w)+(P_e-P_w)S=0, \quad P_e^2=P_e, \quad P_w^2=P_w, \quad P_eP_w=0, 
    \]
    and 
    \[
        S^2=\|w\|^2(P_e+P_w)=\|w\|^2(P_e-P_w)^2.
    \]}
    {\[
        S(P_e-P_w)+(P_e-P_w)S=0, \quad P_e^2=P_e, \quad P_w^2=P_w, \quad P_eP_w=0, \quad \text{and}\quad S^2=\|w\|^2(P_e+P_w)=\|w\|^2(P_e-P_w)^2.
    \]}
    As such, 
    \[
        U^2=\frac{\alpha^2(P_e+P_w)+\|w\|^2(P_e+P_w)}{\|b\|^2}=\frac{\|u\|^2}{\|b\|^2}(P_e+P_w),
    \]
    which implies that, since $U$ is symmetric,
    \[
        \|U\|_{\op}=\sqrt{\|U^2\|_{\op}}=\frac{\|u\|}{\|b\|}\sqrt{\|P_e+P_w\|_{\op}}.
    \]
    Using the definition of eigenvalues, one deduces that $\|P_e+P_w\|_{op}=\lambda_{\max}(P_e+P_w)=1$, and hence all properties are fulfilled.
\end{proof}

\begin{lemma}[{\thmcite[Claim 23]{kifer_private_2012}}]\label{S12::lem:weird_convex}
    Let $g\colon \R^d\to \R$ be a convex function, let $x, y\in \R^d$ and let $\lambda\ge 1$. Then it holds that 
    \[
        \frac{g(x+\lambda y)-g(x)}{\lambda}\ge g(x+y)-g(x).
    \]
\end{lemma}
\begin{proof}
    This follows immediately by the definition of convexity as 
    \[
        g(x+y)=g\left(\frac{1}{\lambda}(x+\lambda y)+\left(1-\frac{1}{\lambda}\right)x\right)\le \frac{1}{\lambda}g(x+\lambda y)+\left(1-\frac{1}{\lambda}\right)g(x),
    \]
    which rearranges to the wanted inequality.
\end{proof}

\section{Wishart Ensemble - Proof of Lemma \ref{lem:satisfies}}

We denote by $\Wishart(d, m)$ the distribution of a $d\times d$ Wishart matrix with hidden dimension $m$. By this we mean that if $G\in \R^{d\times m}$ has independent entries $G_{i, j}\sim \mathcal N(0, 1)$, then $W=GG^T\sim \Wishart(d, m)$. We define $q_{d, m}\colon \S^d_+\to \R_+$ to be the probability density function of $\Wishart(d, m)$, which we will interchangeably denote by $q_{d, m}\colon \R^{d(d+1)/2}\to \R_+$ when $\S^d_+$ has been canonically identified with $\R^{d(d+1)/2}$. It is known that \citep{gupta_matrix_2000}
\begin{equation}\label{eq:wishart_denisty}
    q(W)=\frac{1}{2^{md/2}\Gamma_d(m/2)}\det(W)^{\frac{m-d-1}{2}}\exp(-\tr(W)/2),
\end{equation}
where $\Gamma_d$ is the multivariate Gamma function given by 
\[
    \Gamma_d(x)=\pi^{d(d-1)/4}\prod_{i=1}^d\Gamma(x+(1-i)/2).
\]
Before showing that the Wishart distribution satisfies Assumption \ref{ass:distribution}, we state one auxiliary lemma. 

\begin{lemma}[Tail Bound on Smallest Eigenvalue]\label{lem:tailbound_smallest_eig}
    Let $W\sim \Wishart(d, m)$ be a $d\times d$ Wishart matrix with hidden dimension $m>d$, and let $\lambda_{\min}(W)$ denote the (random) smallest eigenvalue of $W$. Then the marginal density $R(\mu)$ of $\lambda_{\min}(W)$ satisfies
    \[
        R(\mu)\le D\cdot 2^{-1}\cdot (\mu/2)^{\frac{m-d-1}{2}}\cdot e^{-\mu/2}\quad \text{for $D=\frac{d\cdot \Gamma\left(\frac{3}{2}\right)\Gamma\left(\frac{m+1}{2}\right)}{\Gamma\left(\frac d2+1\right)\Gamma\left(\frac{m-d+1}{2}\right)\Gamma\left(\frac{m-d+2}{2}\right)}$},
    \]
    where $\Gamma$ is the (complete) Gamma function. Specifically, we have that, for $s, s_1, s_2\ge 0$ satisfying $s_1\le s_2$,
    \[
        \P(\lambda_{\min}(W)\le s)\le D\cdot \gamma\left(\frac{s}{2}, \frac{m-d+1}{2}\right),
    \]
    and
    \[
        \P(s_1\le \lambda_{\min}(W)\le s_2)\le D\cdot \left[\gamma\left(\frac{s_2}{2}, \frac{m-d+1}{2}\right)-\gamma\left(\frac{s_1}{2}, \frac{m-d+1}{2}\right)\right],
    \]
    where $\gamma$ is the lower incomplete Gamma function.
\end{lemma}
\begin{proof}
    The bound on the marginal density follows from \cite[Lemma 4.7]{deift_conjugate_2021}. The rest follows trivially.
\end{proof}

We are now ready to show that the Wishart distribution satisfies Assumption \ref{ass:distribution}, which extends Lemma \ref{lem:satisfies} by additionally providing the exact constants. 

\begin{lemma}[Wishart Distribution Satisfies Random Matrix Assumption]\label{lem:WishartSatisfiesAss}
    Let $W\sim \Wishart(d, m)$ be a $d\times d$ Wishart matrix with hidden dimension $m>\red{d+1}$. Then the Random Matrix Assumption \ref{ass:distribution} is satisfied with
    \[
        \alpha\le 2 \gamma^{-1}\left(\frac{\delta_3}{D}, p\right), \quad \alpha_1\le 2\gamma^{-1}\left(\frac{\delta_1\cdot (1-\delta_3)}{D}+\gamma\left(\alpha/2, p\right), p\right)-\alpha,
    \]
    \[
        f(\rho)=\left(\frac{p-1}{\alpha}+\frac12\right)\cdot \rho,\quad \mu=m,\quad \red{\nu=\frac{1}{m-d-1},}
    \]
    where $\gamma^{-1}$ is the inverse lower incomplete Gamma function (in its first entry) and $p=\frac{m-d+1}{2}$.
\end{lemma}
\begin{proof}
    The choice of $\mu=m$ is valid as $\Exp{W}=m I$\red{, and the choice of $\nu$ is valid by \cite{vonrosen_moments_1997}}. The existence and validity of choices of $\alpha$ and $\alpha_1$ follow by Lemma \ref{lem:tailbound_smallest_eig}. 
    We focus the rest of the proof on the validity of $f$.

    Since $W\sim \Wishart(d, m)$, we recall Equation \eqref{eq:wishart_denisty} stating that
    \[
        q(W)=\frac{1}{2^{md/2}\Gamma_d(m/2)}\det(W)^{\frac{m-d-1}{2}}\exp(-\tr(W)/2),
    \]
    such that, provided $W+U\in \S_+^d$,
    \[
        \frac{q(W+U)}{q(W)}=\left(\frac{\det(W+U)}{\det(W)}\right)^{\frac{m-d-1}2}\cdot \exp(\tr(-U)/2).
    \]
    The matrix $U$ is symmetric with rank $\rank(U)$ and eigenvalues bounded by $\|U\|_{\op}$, and hence, by the Low-Rank Update Bound \ref{lem:lowrankupdate},
    \[
        \frac{\det(W+U)}{\det(W)}\le \exp\left(\frac{\|U\|_{\op}}{\alpha}\cdot \rank(U)\right),
    \]
    where the last inequality follows by the lower bound on $\lambda_{\min}(W)$. Moreover, $\tr(-U)\le \|U\|_{\op}\cdot \rank(U)$, and thus, 
    \[
        \frac{q(W+U)}{q(W)}\le \exp\left(\|U\|_{\op}\left(\frac{m-d-1}{2\alpha}+\frac12\right)\cdot \rank(U)\right).
    \]
    As such, the assumption is verified with the given $f$.
\end{proof}
\section{Privacy Analysis - Proof of Theorem \ref{thm:main}}\label{sec::proof_dp}

We assume that, additionally to releasing a noisy version of the approximate minimizer $\theta_{\final}=\theta_\approx+b$ in Mechanism \ref{alg:QOP}, we also release \red{a perturbed version of the gradient of the at-center-loss and} the exact private minimizer
\[
    \red{Y=\sum_{i=1}^n\nabla \ell(\theta_c; z_i)+a\quad\text{and}\quad} \theta_\priv=\argmin_{\theta\in C} \{\Jp(\theta; \D)\}.
\]
Note that releasing additional information can only worsen the privacy guarantees, and hence the privacy accounting for the mechanism with additional release must also be valid for the mechanism itself. As such, we study the mechanism releasing \red{$Y$, }$\theta_\final$ and $\theta_\priv$. From a privacy perspective, this is equivalent to releasing \red{$Y$,} $\theta_\priv$ and $\theta_\final-\theta_\priv$.

\red{The motivation behind releasing $Y$ is that, given a realization $Y=y$, it holds that
\[
    \Jp(\theta; \D)=\sum_{i=1}^n\bar\ell(\theta; z_i)+\bar r(\theta)+\frac{\sigma^2}{2}(\theta-\theta_c)^TW(\theta-\theta_c),
\]
where $\bar \ell(\theta; z)=\ell(\theta; z)-\langle \nabla \ell(\theta_c; z), \theta-\theta_c\rangle$ and $\bar r(\theta)=r(\theta)+\langle y, \theta-\theta_c\rangle$. Specifically, we obtain that $\theta_c\in \argmin \bar \ell(\cdot; z)$ for all $z\in \X$, and $\bar\ell$ still satisfies all the same assumptions as $\ell$ in Problem Assumption \ref{ass:problem}. Moreover, $\bar r$ remains convex, and therefore also satisfies the same assumptions as $r$ in Problem Assumption \ref{ass:problem}. The key addition is that all $\bar \ell(\cdot, z)$ share the common minimizer $\theta_c$ after this transformation.
}

Given $\alpha>0$ from the Random Matrix Assumption \ref{ass:distribution}, we define the events
\[
    E_\alpha=\{W\in \S_+^d\colon \lambda_{\min}(W)\ge \alpha\}.
\]
We will show that \red{the release of $Y$ is $(\varepsilon_0, \delta_0)$-DP, that} the release of $\theta_\priv$, conditioned on $E_\alpha$, is $(\varepsilon_1, \delta_1)$-DP\red{, uniformly in the realization of $Y\in \R^d$}, and that the release of $\theta_\final-\theta_\priv$, conditioned on $E_{\alpha}$, is $(\varepsilon_2, \delta_2)$-DP, uniformly in the realizations of \red{$Y\in \R^d$ and} $W\in E_{\alpha}$. By applying the Composition of Mechanisms Lemma \ref{lem:composition}, this shows the joint release of $(\red{Y, }\theta_\priv, \theta_\final)$, and hence also Mechanism \ref{alg:QOP}, conditioned on $E_\alpha$, is $(\red{\varepsilon_0+}\varepsilon_1+\varepsilon_2, \red{\delta_0+}\delta_1+\delta_2)$-DP. These results are established in Lemmas \red{\ref{lem:part0}, }\ref{lem:part1} and \ref{lem:part2}.

Since $\P(E_\alpha)\ge 1-\delta_3$, this means that the unconditional version of Mechanism \ref{alg:QOP} will be $(\red{\varepsilon_0+}\varepsilon_1+\varepsilon_2, \red{\delta_0+}\delta_1+\delta_2+\delta_3)=(\varepsilon, \delta)$-DP, as wanted.

\subsection{Release of Exact Minimizer}

We initially prove the privacy of releasing $\theta_\priv$ under the following simplification assumption. We will show this assumption to be unnecessary in Lemma \ref{lem:part1}.

\begin{assumption}[Simplification Assumption]\label{ass:simplification}
    We make the following two simplification assumptions:
    \begin{enumerate}
        \item the regularizer $r\colon \R^d\to \R$ is twice \red{continuously} differentiable; and 
        \item the convex set $C=\R^d$ is the full domain.
    \end{enumerate}
\end{assumption}

\red{We note that under Simplification Assumption \ref{ass:simplification}, it also holds that $\bar r$ is twice differentiable.}

\red{
\begin{lemma}[\red{Virtual Release of Auxiliary Quantity}]\label{lem:part0}\label{lem:releaseY}
    \red{
    Assume that the Problem Assumption \ref{ass:problem} holds.
    The release of $Y(\D)=\sum_{i=1}^n\nabla \ell(\theta_c; z_i)+a$ is $(\varepsilon_0, \delta_{0})$-DP for 
    \[
        \sigma_0\ge \frac{\Delta_c}{\varepsilon_0}\sqrt{2\log\frac{1.25}{\delta_0}}.
    \]}
\end{lemma}
\begin{proof}
    This follows by the Gaussian Mechanism \ref{lem:gauss}, since
    \[
        \sup_{\mathcal D\sim \mathcal D'}\left\|\sum_{i=1}^n\nabla \ell(\theta_c; z_i)-\sum_{i=1}^n\nabla \ell(\theta_c; z'_i)\right\|=\sup_{z, z'\in \X}\|\nabla \ell(\theta_c; z)-\nabla \ell(\theta_c; z')\|\le \Delta_c. \qedhere
    \]
\end{proof}
}

\begin{lemma}[Conditional Release of Exact Minimizer under Simplification Assumption]\label{lem:part1_simple}
    Assume that the Problem Assumption \ref{ass:problem}, the Random Matrix Assumption \ref{ass:distribution} and the Simplification Assumption \ref{ass:simplification} hold.
    The release of $\theta_\priv$, conditioned on the events $E_\alpha$ \red{and $\{Y=y\}$}, is $(\varepsilon_1, \delta_{1})$-DP, \red{uniformly in the realization of $Y$,} for 
    \[
        \sigma^2\ge \max\left(\frac{2L}{\varepsilon_1}\left(f(2)+\frac{2(2\rho+2)}{\alpha}\right), \frac{2L}{\alpha_1}\right)
    \]
    where $\alpha$, $\alpha_1$ and $f$ are given by the Random Matrix Assumption \ref{ass:distribution}.
\end{lemma}
\begin{proof}
    \let\oldell\ell\renewcommand{\ell}{\red{\bar\oldell}}
    \let\oldL\L\renewcommand{\L}{\red{\bar\oldL}}
    
    For ease of notation we define $\L(\theta;\D)=\sum_{i=1}^n\ell(\theta; z_i)$. We note that the objective $\Jp(~\cdot~; \D)$ is $\alpha\sigma^2$-strongly convex under $E_\alpha$. Specifically, given a noise $W\in E_\alpha$, we can determine $\theta_\priv$ uniquely through the map
    \[
        \theta_\priv=T_\D(W)=\argmin_{\theta\in \R^d} \left\{\red{\bar\J_\priv(\theta;\D)=}\L(\theta;\D)+\red{\bar r}(\theta)+\frac{\sigma^2}{2}(\theta-\red{\theta_c})^TW(\theta-\red{\theta_c})\right\},
    \]
    whose fiber is characterized by 
    \[
        T_\D^{-1}(\theta)=\left\{W\in E_\alpha\subset \S^d_+\colon W(\theta-\red{\theta_c})=u_\D(\theta)\right\},\quad \text{where~} u_\D(\theta)=-\sigma^{-2}\left(\nabla \L(\theta;\D)+\nabla \red{\bar r}(\theta)\right).
    \]
    \red{Note that $\bar\J_\priv$ and $\Jp$ are equal up to an additive constant, and therefore share the same minimizers.}
    
    Observe that if $\red{\theta_c}\in \argmin \red{\bar r}$, then $\theta_\priv=\red{\theta_c}$ under $E_\alpha$, and hence releasing $\theta_\priv$ is perfectly private, and hence certainly also $(\varepsilon_1, \delta_{1})$-DP. As such, from now on, assume that $\red{\theta_c}\not\in\argmin \red{\bar r}$. We note that $\nabla \Jp(\red{\theta_c}; \D)=\nabla \red{\bar r}(\red{\theta_c})\neq 0$ in that case, and hence $\red{\theta_c}\not\in \argmin \Jp$, and hence $T_\D(W)\neq \red{\theta_c}$. 
    
    This proof is split into multiple parts. In Part 1, we derive the conditional density of the distribution of $\theta=T_\D(W)$ under $E_\alpha$. In Part 2, we derive an upper bound on the ratio of this density.
    
    \textbf{Part 1. Conditional Density}
    
    Let us compute the conditional density $\red{p_\D^y}(\theta|E_\alpha)$ of $\theta=T_\D(W)$ under $E_\alpha$. By Lemma \ref{lem:transport}, it holds that $\theta=T_\D(W)$ follows a distribution with density $\red{p_\D^y}$ given by
    \tight{\begin{align}
        \red{p_\D^y}(\theta|E_\alpha)
        &=\int_{T_\D^{-1}(\theta)}\frac{q(W|E_\alpha)}{\|J_dT_\D(W)\|}d\H^{d(d+1)/2-d}(W) \nonumber\\
        &=\frac{1}{\P(E_\alpha)}\int_{T_\D^{-1}(\theta)\cap E_\alpha}\frac{q(W)}{\|J_dT_\D(W)\|}d\H^{\frac{d(d-1)}{2}}(W),\label{eq:density}
    \end{align}}
    {\begin{equation}\label{eq:density}
        \red{p_\D^y}(\theta|E_\alpha)=\int_{T_\D^{-1}(\theta)}\frac{q(W|E_\alpha)}{\|J_dT_\D(W)\|}d\H^{d(d+1)/2-d}(W)=\frac{1}{\P(E_\alpha)}\int_{T_\D^{-1}(\theta)\cap E_\alpha}\frac{q(W)}{\|J_dT_\D(W)\|}d\H^{\frac{d(d-1)}{2}}(W),
    \end{equation}
    }where we canonically identified $\S^d$ with $\R^{d(d+1)/2}$. We note that we require $\|J_dT_\D(W)\|>0$ almost everywhere on $\{q(W|E_\alpha)>0\}$ and $T_\D$ to be locally Lipschitz continuous under $E_\alpha$, which we shall show later. In order to evaluate $\red{p_\D^y}(\theta|E_\alpha)$, we shall compute $\|J_dT_\D(W)\|$. First note that $\theta=T_\D(W)$ if, and only if, 
    \[
        F_\D(\theta, W)=\nabla \L(\theta;\D)+\nabla \red{\bar r}(\theta)+\sigma^2W(\theta-\red{\theta_c})=0.
    \]
    As such, along the curve $\theta=T_\D(W)$, we have $F_\D(\theta, W)\equiv 0$. We have that
    \[
        D_\theta F_\D(\theta, W)[\delta \theta]=(\nabla ^2\L(\theta;\D)+\nabla^2 \red{\bar r}(\theta)+\sigma^2W)(\delta \theta),
    \]
    where we note the inverse is well-defined under $E_\alpha$, and 
    \[
        D_W F_\D(\theta, W)[\delta W]=\sigma^2(\delta W)(\theta-\red{\theta_c}),
    \]
    and as such, by the Implicit Function Theorem, since $D_\theta F_\D$ is invertible under $E_\alpha$,
    \[
        D_WT_\D(W)[\delta W]=-\sigma^2(\nabla ^2\L(T_\D(W);\D)+\nabla^2 \red{\bar r}(T_\D(W))+\sigma^2W)^{-1}(\delta W)(T_\D(W)-\red{\theta_c}).
    \]
    Moreover, we know that $T_\D$ is locally $C^1$, and hence locally Lipschitz-continuous. Specifically, if we define 
    \[
        A_W\colon \R^d\to \R^d, \quad v\mapsto -\sigma^2(\nabla ^2\L(T_\D(W);\D)+\nabla^2 \red{\bar r}(T_\D(W))+\sigma^2W)^{-1} v,
    \]
    and 
    \[
        B_\theta\colon \S^d_+\to \R^d, \quad \delta W\mapsto (\delta W)(\theta-\red{\theta_c}),
    \]
    then we can write 
    \[
        DT_\D(W)=A_W\circ B_{T_\D(W)}.
    \]
    Now we compute the adjoint operator $B_{\theta}^T$ by noting that
    \tight
    {\begin{align*}
        \langle B_\theta(\delta W), v\rangle 
        &= (\theta-\red{\theta_c})^T(\delta W)v
        =\tr((\delta W) (\theta-\red{\theta_c}) v ^T)
        =\langle \delta W, (\theta-\red{\theta_c}) v^T\rangle \\
        &=\left\langle \delta W, \frac{(\theta-\red{\theta_c}) v^T+v(\theta-\red{\theta_c})^T}{2}\right\rangle,
    \end{align*}}
    {\[
        \langle B_\theta(\delta W), v\rangle 
        = (\theta-\red{\theta_c})^T(\delta W)v
        =\tr((\delta W) (\theta-\red{\theta_c}) v ^T)
        =\langle \delta W, (\theta-\red{\theta_c}) v^T\rangle
        =\left\langle \delta W, \frac{(\theta-\red{\theta_c}) v^T+v(\theta-\red{\theta_c})^T}{2}\right\rangle,
    \]}
    and hence $B_\theta^T(v)=\sym((\theta-\red{\theta_c}) v^T)$. As such, 
    \[
        B_\theta\circ B_\theta^T(v)=\frac{(\theta-\red{\theta_c}) v^T+v(\theta-\red{\theta_c})^T}{2}(\theta-\red{\theta_c})=\frac{(\theta-\red{\theta_c})(\theta-\red{\theta_c})^T+\|\theta-\red{\theta_c}\|^2}{2}v,
    \]
    or, in short-hand, 
    \[
        B_\theta B_\theta^T=\frac{(\theta-\red{\theta_c})(\theta-\red{\theta_c})^T+\|\theta-\red{\theta_c}\|^2I}{2}.
    \]
    Specifically, using the Rank-1 Update Formula \ref{lem:rank1update}, we obtain
    \[
        \det\left(B_\theta B_\theta^T\right)=2^{-d}\left(1+\frac{(\theta-\red{\theta_c})^T(\theta-\red{\theta_c})}{\|\theta-\red{\theta_c}\|^{2}}\right)\det(\|\theta-\red{\theta_c}\|^2I)=2^{1-d}\|\theta-\red{\theta_c}\|^{2d}.
    \]
    Moreover, we note that $A_W$ is symmetric, and hence, with $\theta=T_\D(W)$,
    \begin{align*}
        \det\left(DT_\D(W)DT_\D(W)^T\right)
        &=\det(A_W B_{T_\D(W)} B_{T_\D(W)}^T A_W^T)\\
        &=\det(A_W)^2\det(B_{T_\D(W)} B_{T_\D(W)}^T)\\
        &=\frac{\sigma^{4d}\cdot 2^{1-d}\|\theta-\red{\theta_c}\|^{2d}}{\det(\nabla ^2\L(T_\D(W);\D)+\nabla ^2\red{\bar r}(T_\D(W))+\sigma^2W)^2}.
    \end{align*}
    If we define 
    \[
        \bar A_\D(W)=\nabla ^2\L(T_\D(W);\D)+\nabla ^2\red{\bar r}(T_\D(W))+\sigma^2W,
    \]
    then we can write 
    \[
        \|J_dT_\D(W)\|=\sqrt{\det\left(DT_\D(W)DT_\D(W)^T\right)}=\frac{2^{\frac{1-d}{2}}\cdot \|\theta-\red{\theta_c}\|^d\cdot \sigma^{2d}}{|\det(\bar A_\D(W))|}=\frac{2^{\frac{1-d}{2}}\cdot \|\theta-\red{\theta_c}\|^d\cdot \sigma^{2d}}{\det(\bar A_\D(W))},
    \]
    where the final equality follows from $\bar A_\D(W)\succeq 0$. Plugging this into Equation \eqref{eq:density} yields 
    \[
        \red{p_\D^y}(\theta|E_\alpha)=\frac{1}{P(E_\alpha)\cdot 2^{\frac{1-d}{2}}\cdot \|\theta-\red{\theta_c}\|^d\cdot \sigma^{2d}}\int_{T_\D^{-1}(\theta)\cap E_\alpha}q(W)\det(\bar A_\D(W))d\H^{\frac{d(d-1)}{2}}(W).
    \]
    For simplicity of notation, we denote the data-independent scaling constant by $C_\alpha(\theta)$. We do note that $\bar A_\D(W)\succeq \sigma^2W\succ 0$ and $\|\theta-\red{\theta_c}\|>0$, and hence $\|J_dT_\D(W)\|>0$ almost everywhere on $\{q(W|E_\alpha)>0\}$, as required for Lemma \ref{lem:transport}.
    
    \textbf{Part 2. Upper Bound on Ratios of Densities}
    
    Now take two neighboring datasets $\D$ and $\D'$ of size $n$, which differ only in the last datapoint (see Definition \ref{def:adjacency}). Define 
    \[
        E=\nabla^2\ell(\theta; z'_n)-\nabla^2\ell(\theta; z_n),
    \]
    such that, for all $W\in T_\D^{-1}(\theta)$ and $W'\in T_{\D'}^{-1}(\theta)$, it holds that
    \[
        \bar A_{\D'}(W')-\sigma^2 W'=\bar A_\D(W)-\sigma^2 W+ E.
    \]
    Recall that the fiber $T_\D^{-1}(\theta)$ is described by $W(\theta-\red{\theta_c})=u_\D(\theta)$, and note that
    \[
        u_{\D'}(\theta)=u_\D(\theta)-\sigma^{-2}\delta u,
    \]
    where $\delta u=\nabla \ell(\theta; z'_n)-\nabla \ell(\theta; z_n)$. Now define $U$ such that $\|U\|_{op}=\frac{\|\delta u\|}{\|\theta-\red{\theta_c}\|}$, such that $U(\theta-\red{\theta_c})=\delta u$, and such that $\rank(U)\le 2$. This is possible by Lemma \ref{lem:matrixU}.
    We note that by $L$-smoothness, we have that 
    \[
        \|U\|_{\text{op}}=\frac{\|\delta u\|}{\|\theta-\red{\theta_c}\|}\le \frac{\|\nabla \ell(\theta; z'_n)-\nabla \ell(\red{\theta_c}; z'_n)\|}{\|\theta-\red{\theta_c}\|}+\frac{\|\nabla \ell(\theta; z_n)-\nabla \ell(\red{\theta_c}; z_n)\|}{\|\theta-\red{\theta_c}\|}\le 2L.
    \]
    As such, $(W, W')\in T_\D^{-1}(\theta)\times T_{\D'}^{-1}(\theta)$, if, and only if, 
    \[
        W'(\theta-\red{\theta_c})=W(\theta-\red{\theta_c})-\sigma^{-2} U(\theta-\red{\theta_c}).
    \]
    Define the translation $\phi(W)=W+ \sigma^{-2} U$. Specifically, if $W'\in T_{\D'}^{-1}(\theta)\cap E_\alpha$, then either $\phi(W')\in T_{\D}^{-1}(\theta)\cap E_\alpha$ or $\phi(W')\not\in E_\alpha$. As such, we may define 
    \tight{
        \[G_{\D'}(\theta)=\{W'\in T_{D'}^{-1}(\theta)\cap E_\alpha\colon \phi(W')\in T_D^{-1}(\theta)\cap E_\alpha\}\]
        and 
        \[B_{\D'}(\theta)=\{W'\in T_{D'}^{-1}(\theta)\cap E_\alpha\colon \phi(W')\not\in E_\alpha\},\]
    }
    {\[
        G_{\D'}(\theta)=\{W'\in T_{D'}^{-1}(\theta)\cap E_\alpha\colon \phi(W')\in T_D^{-1}(\theta)\cap E_\alpha\}\quad \text{and}\quad 
        B_{\D'}(\theta)=\{W'\in T_{D'}^{-1}(\theta)\cap E_\alpha\colon \phi(W')\not\in E_\alpha\},
    \]}
    such that $\phi(G_{\D'}(\theta))\subset T_\D^{-1}(\theta)\cap E_\alpha$, and $T_{\D'}^{-1}(\theta)\cap E_\alpha=G_{\D'}(\theta)\uplus B_{\D'}(\theta)$. As such, 
    \tight{\begin{align*}
        \red{p_{\D'}^y}(\theta|E_\alpha)
        &=C_\alpha(\theta)\cdot \int_{T_{\D'}^{-1}(\theta)\cap E_\alpha}q(W')\det(\bar A_{\D'}(W'))d\H^{\frac{d(d-1)}{2}}(W') \\
        &=\underbrace{C_\alpha(\theta)\int_{G_{\D'}(\theta)}q(W')\det(\bar A_{\D'}(W'))d\H^{\frac{d(d-1)}{2}}(W')}_{\text{(Part 2a)}} \\
        &\qquad +  \underbrace{C_\alpha(\theta)\int_{B_{\D'}(\theta)}q(W')\det(\bar A_{\D'}(W'))d\H^{\frac{d(d-1)}{2}}(W') }_{\text{(Part 2b)}}.
    \end{align*}}
    {\begin{align*}
        \red{p_{\D'}^y}(\theta|E_\alpha)
        &=C_\alpha(\theta)\cdot \int_{T_{\D'}^{-1}(\theta)\cap E_\alpha}q(W')\det(\bar A_{\D'}(W'))d\H^{\frac{d(d-1)}{2}}(W') \\
        &=\underbrace{C_\alpha(\theta)\int_{G_{\D'}(\theta)}q(W')\det(\bar A_{\D'}(W'))d\H^{\frac{d(d-1)}{2}}(W')}_{\text{(Part 2a)}} +  \underbrace{C_\alpha(\theta)\int_{B_{\D'}(\theta)}q(W')\det(\bar A_{\D'}(W'))d\H^{\frac{d(d-1)}{2}}(W') }_{\text{(Part 2b)}}.
    \end{align*}}
    By using the translation defined above, we have that, by change of variables,
    \tight{\begin{align*}
        \frac{\text{(Part 2a)}}{C_\alpha(\theta)}
        &= \int_{\phi(G_{\D'}(\theta))}q(\phi^{-1}(W))\det(\bar A_{\D'}(\phi^{-1}(W)))d\H^{\frac{d(d-1)}{2}}(W) \\
        &= \int_{\phi(G_{\D'}(\theta))}q(W-\sigma^{-2}U)\det(\bar A_{\D}(W)- U+ E)d\H^{\frac{d(d-1)}{2}}(W) \\
        &\le e^{2\sigma^{-2}L\cdot f(2)}\cdot \sup_{W\in \phi(G_{\D'}(\theta))}\left\{\frac{\det(\bar A_{\D}(W)- U+ E)}{\det(\bar A_{\D}(W))}\right\} \\
        &\qquad\qquad \cdot \int_{\phi(G_{\D'}(\theta))}q(W)\det(\bar A_{\D}(W))d\H^{\frac{d(d-1)}{2}}(W) \\
        &\le e^{2\sigma^{-2}L\cdot f(2)}\cdot \sup_{W\in E_\alpha}\left\{\frac{\det(\bar A_{\D}(W)- U+ E)}{\det(\bar A_{\D}(W))}\right\} \\
        &\qquad\qquad \cdot \int_{T_\D^{-1}(\theta)\cap E_\alpha}q(W)\det(\bar A_{\D}(W))d\H^{\frac{d(d-1)}{2}}(W) \\
        &= e^{2\sigma^{-2}L\cdot f(2)}\cdot \sup_{W\in E_\alpha}\left\{\frac{\det(\bar A_{\D}(W)- U+ E)}{\det(\bar A_{\D}(W))}\right\}\cdot\frac{p_{\D}(\theta|E_\alpha)}{C_\alpha(\theta)},
    \end{align*}}
    {\begin{align*}
        \frac{\text{(Part 2a)}}{C_\alpha(\theta)}
        &= \int_{\phi(G_{\D'}(\theta))}q(\phi^{-1}(W))\det(\bar A_{\D'}(\phi^{-1}(W)))d\H^{\frac{d(d-1)}{2}}(W) \\
        &= \int_{\phi(G_{\D'}(\theta))}q(W-\sigma^{-2}U)\det(\bar A_{\D}(W)- U+ E)d\H^{\frac{d(d-1)}{2}}(W) \\
        &\le e^{2\sigma^{-2}L\cdot f(2)}\cdot \sup_{W\in \phi(G_{\D'}(\theta))}\left\{\frac{\det(\bar A_{\D}(W)- U+ E)}{\det(\bar A_{\D}(W))}\right\}\cdot \int_{\phi(G_{\D'}(\theta))}q(W)\det(\bar A_{\D}(W))d\H^{\frac{d(d-1)}{2}}(W) \\
        &\le e^{2\sigma^{-2}L\cdot f(2)}\cdot \sup_{W\in E_\alpha}\left\{\frac{\det(\bar A_{\D}(W)- U+ E)}{\det(\bar A_{\D}(W))}\right\}\cdot \int_{T_\D^{-1}(\theta)\cap E_\alpha}q(W)\det(\bar A_{\D}(W))d\H^{\frac{d(d-1)}{2}}(W) \\
        &= e^{2\sigma^{-2}L\cdot f(2)}\cdot \sup_{W\in E_\alpha}\left\{\frac{\det(\bar A_{\D}(W)- U+ E)}{\det(\bar A_{\D}(W))}\right\}\cdot \frac{p_{\D}(\theta|E_\alpha)}{C_\alpha(\theta)},
    \end{align*}}
    where the first inequality follows by assumption. To bound the supremum, we recall that $E=\nabla^2\ell (\theta; z'_n)-\nabla^2\ell (\theta; z_n)$ has rank at most $2\rho$ and all eigenvalues bounded by $2L$, and that $U$ has rank at most $2$ and all eigenvalues bounded by $2L$. We apply the Low-Rank Update Bound \ref{lem:lowrankupdate} with $\bar A_\D(W)\succeq \alpha \sigma^2 I$ on $E_\alpha$ and $|\lambda_i(- U+E)|\le 4L$ for all $1\le i\le 2\rho+2$, where $\rank(- U+E)\le 2\rho+2$, to obtain that
    \[
        \frac{\det(\bar A_\D(W)- U+ {E})}{\det(\bar A_\D(W))}\le \exp\left(\frac{4L(2\rho+2)}{\alpha \sigma^2}\right).
    \]
    As such, we conclude that 
    \[
        (\text{Part 2a})\le \exp\left(2\sigma^{-2}L\cdot f(2)+2\sigma^{-2}L\frac{2(2\rho+2)}{\alpha}\right)\cdot \red{p_\D^y}(\theta|E_\alpha)\le e^{\varepsilon_1}\cdot \red{p_\D^y}(\theta|E_\alpha),
    \]
    where the latter follows from the bound on $\sigma^2$.

    For $\text{(Part 2b)}$, notice that 
    \[
        B_{\D'}(\theta)\subset \bar B=\{W\in \S_+^d\colon \alpha\le \lambda_{\min}(W)\le \alpha+2\sigma^{-2}L\},
    \]
    since $\|U\|_{\text{op}}\le 2L$, and hence 
    \begin{align*}
        (\text{Part 2b})
         \le C_\alpha(\theta)\int_{T_{\D'}(\theta)\cap \bar B}q(W')\det(\bar A_{\D'}(W'))d\H^{\frac{d(d-1)}{2}}(W') 
        = \frac{\P(\bar B)}{\P(E_\alpha)}\red{p_{\D'}^y}(\theta|\bar B).
    \end{align*}
    As such, for any measurable $S\subset \R^d$,
    \[
        \int_{\theta\in S}(\text{Part 2b})d\theta \le \frac{\P(\bar B)}{\P(E_\alpha)}\int_{\theta\in S}\red{p_{\D'}^y}(\theta|\bar B)\le \frac{\P(\bar B)}{\P(E_\alpha)}\le \delta_{1},
    \]
    where the last step follows by assumption and by the bound on $\sigma^2$. Specifically, for any measurable $S\subset \R^d$, it holds that 
    \[
        \red{\P_{\D'}}(\theta\in S|E_\alpha\red{, Y=y})\le e^{\varepsilon_1}\P_{\D}(\theta\in S|E_\alpha\red{, Y=y})+\delta_{1},
    \]
    and hence the conditional mechanism is $(\varepsilon_1, \delta_{1})$-DP, as wanted.
\end{proof}

The result of Lemma \ref{lem:part1_simple} holds without the Simplification Assumption \ref{ass:simplification}. This is detailed in the following. 

\begin{lemma}[Conditional Release of Exact Minimizer]\label{lem:part1}
    Assume that the Problem Assumption \ref{ass:problem} and the Random Matrix Assumption \ref{ass:distribution} hold.
    The release of $\theta_\priv$, conditioned on the event $E_\alpha$ \red{and $\{Y=y\}$}, is $(\varepsilon_1, \delta_{1})$-DP, \red{uniformly in the realization of $Y$,} for 
    \[
        \sigma^2\ge \max\left(\frac{2L}{\varepsilon_1}\left(f(2)+\frac{2(2\rho+2)}{\alpha}\right), \frac{2L}{\alpha_1}\right)
    \]
    where $\alpha$, $\alpha_1$ and $f$ are given by the Random Matrix Assumption \ref{ass:distribution}.
\end{lemma}
\begin{proof}
    The proof relies on the Successive Approximation Lemma \ref{lem:sucessive}, following \cite{kifer_private_2012}. In fact, the proof closely follows the arguments laid out in \cite{kifer_private_2012}. The two assumptions we aim to prove are the twice-differentiability of the regularizer (Part 1) and the full-domain assumption of the constraint set (Part 2).

    \textbf{Part 1. Removal of Twice-Differentiability Assumption of the Regularizer.}

    In this part, we still assume the constraint set $C$ to be the full domain $C=\R^d$. For this, we consider the function $\psi\colon \R\to \R$ defined as 
    \[
        \psi(x)=\begin{cases}
            \exp(-\frac{1}{1-x^2})\quad &{\text{if $|x|<1$}} \\
            0 & \text{else}.
        \end{cases}
    \]
    This function, known as the bump function, is $C^\infty$ and all its derivatives vanish outside $(-1, 1)$. Through this, we define the sequence of kernel functions $K_i\colon \R^{d}\to \R$ defined as 
    \[
        K_i(\theta)=\frac{\psi(i\|\theta\|^2)}{\int_{\R^d}\psi(i\|\theta'\|^2)d\theta'}.
    \]
    As such, $K_i$ is also $C^\infty$ and all its derivatives vanish outside its support $\supp(K_i)=\{\theta\colon \|\theta\|\le 1/\sqrt i\}$. Given the non-differentiable regularizer $r$, we define the approximate regularizer $r_i$ as 
    \[
        r_i(\theta)=[r\star K_i](\theta)=\int_{\R^d}r(\theta-\theta')K_i(\theta')d\theta'.
    \]
    As $K_i\in C^\infty$, it also holds that $r_i\in C^\infty$, and as $r$ is convex it also holds that $r_i$ is convex. Now define $\Jp$ to be the private objective with the regularizer $r$, and $\Jpi$ to be the private objective with the regularizer $r_i$. We know that the release of $\phi^i(W)$ is $(\varepsilon_1, \delta_1)$-DP by Lemma \ref{lem:part1_simple}. We shall thus show that
    \[
        \lim_{i\to \infty}\phi^i(W)=\phi(W) \quad \text{for each realization $W\in \E_\alpha$},
    \]
    where $\phi^i$ is the unique minimizer of $\Jpi(~\cdot~; \D)$ and $\phi$ is the unique minimizer of $\Jp(~\cdot~; \D)$, and thus conclude that the release of $\phi(W)$ is also $(\varepsilon_1, \delta_1)$-DP by the Successive Approximation Lemma \ref{lem:sucessive}.
    We shall split the remainder of this part into two; proving uniform convergence of $\Jpi(~\cdot~; \D)$ towards $\Jp(~\cdot~; \D)$ for each $W\in E_\alpha$ (Part 1a) and therefrom proving pointwise convergence of $\phi^i$ towards $\phi$ (Part 1b).

    \textbf{Part 1a.} We consider a bounded set $\mathcal I\subset \R^d$ and an arbitrary $W\in E_\alpha$, and show that $\Jpi(~\cdot~; \D)$ converges uniformly to $\Jp(~\cdot~; \D)$ on $\mathcal I$. Since $r$ is continuous over the compact $\bar {\mathcal I}+B(0, 1)$, it is also uniformly continuous over $\bar {\mathcal I}+B(0, 1)$. As such, given a $\xi>0$, there exists an $\eta>0$ such that, if $\|\theta_1-\theta_2\|\le \eta$ for $\theta_1, \theta_2\in \mathcal I+B(0, 1)$, then $\|r(\theta_1)-r(\theta_2)\|\le \xi$. Select $i>\max(1, 1/ \eta^2)$, and any $\theta\in \mathcal I$, such that 
    \begin{align*}
        |\Jpi(\theta; \D)-\Jp(\theta; \D)|
        & =|r_i(\theta)-r(\theta)| \\
        & = \left|\int_{\R^d} r(\theta-\theta')K_i(\theta')d\theta'-r(\theta)\right| \\
        & = \left|\int_{\R^d} [r(\theta-\theta')-r(\theta)]K_i(\theta')d\theta'\right| \\
        & \le \int_{\R^d} \left|r(\theta-\theta')-r(\theta)\right|K_i(\theta')d\theta' \\
        & = \int_{\|\theta'\|\le 1/\sqrt i} \left|r(\theta-\theta')-r(\theta)\right|K_i(\theta')d\theta' \\
        &\le \xi \int_{\|\theta'\|\le 1/\sqrt i} K_i(\theta')d\theta' \\
        &= \xi.
    \end{align*}
    As such, $\Jpi(~\cdot~; \D)$ converges uniformly to $\Jp(~\cdot~; \D)$ on $\mathcal I$.

    \textbf{Part 1b.} Fix any $W\in E_\alpha$, choose any $0\le \xi< \frac{\alpha\sigma^2}{4}$, and define $\mathcal I=\{\theta\colon \|\theta-\phi(W)\|^2\le 1\}$. By Part 1a we know that $\Jpi(~\cdot~; \D)$ converges uniformly to $\Jp(~\cdot~; \D)$ on $\mathcal I$, and hence we may choose an $i_\xi$ such that, for all $i\ge i_\xi$, $|\Jpi(~\cdot~; \D)-\Jp(~\cdot~; \D)|\le \xi$ uniformly on $\mathcal I$. Now note that by $\alpha\sigma^2$-strong convexity of $\Jp(~\cdot~; \D)$, we have that
    \[
        \frac{\alpha\sigma^2}{2}\|\theta-\phi(W)\|^2\le \Jp(\theta; \D)-\Jp(\phi(W); \D)\quad \text{for all $\theta\in \R^d$}.
    \]
    Specifically, if $\theta\in \partial\mathcal I$, it holds that $\Jp(\theta; \D)-\Jp(\phi(W); \D)\ge \frac{\alpha\sigma^2}{2}$, and by uniform convergence since $\phi(W), \theta\in \mathcal I$ we thus get that 
    \[
        \Jpi(\theta; \D)-\Jpi(\phi(W); \D)\ge \Jp(\theta; \D)-\Jp(\phi(W); \D)-2\xi\ge \frac{\alpha\sigma^2}{2}-2\xi>0.
    \]
    We now show that $\phi^i(W)\in \mathcal I$. Assuming the negative, namely that $\|\phi^i(W)-\phi(W)\|>1$, we define 
    \[
        \tilde \phi^i=\phi(W)+\frac{\phi^i(W)-\phi(W)}{\|\phi^i(W)-\phi(W)\|}=\left(1-\lambda\right)\phi(W)+\lambda\phi^i(W)\in \partial \mathcal I,
    \]
    for $\lambda=\frac{1}{\|\phi^i(W)-\phi(W)\|}$. As such, by convexity of $\Jpi(~\cdot~; \D)$, and since $\lambda\in (0, 1)$, we get
    \[
        \Jpi(\phi(W); \D)<\Jpi(\tilde \phi^i; \D)\le (1-\lambda)\Jpi(\phi(W); \D)+\lambda\Jpi(\phi^i(W); \D)\le \Jpi(\phi(W); \D),
    \]
    where the first inequality follows from the above and the last inequality follows from $\phi^i(W)$ being the minimizer of $\Jpi(~\cdot~; \D)$. The above results in a contradiction, thus proving that we must have $\phi^i(W)\in \mathcal I$. As such, we obtain that
    \begin{align*}
        \frac{\alpha\sigma^2}{2}\|\phi^i(W)-\phi(W)\|^2
        &\le \Jp(\phi^i(W); \D)-\Jp(\phi(W); \D) \\
        &\le \Jpi(\phi^i(W); \D)-\Jpi(\phi(W); \D)+2\xi \\
        &\le 2\xi.
    \end{align*}
    As such, $(\phi^i(W))$ converges to $\phi(W)$, and thus $(\phi^i)$ converges pointwise to $\phi$, as wanted.

    \textbf{Part 2. Removal of Full-Domain Assumption of the Constraint Set.}

    Consider the sequence of functions, indexed by $i\ge 0$, 
    \[
        \Jpi(\theta; \D)=\Jp(\theta; \D)+i\dist(\theta, C).
    \]
    We denote by $\phi(W)=\argmin_{\theta\in C}\Jp(\theta; \D)$, and by $\phi^i(W)=\argmin_{\theta\in \R^d}\Jpi(\theta; \D)$ (note the different domains), and shall prove that $\phi^i$ converges pointwise to $\phi$. Since we know that the release of $\phi^i(W)$ is $(\varepsilon_1, \delta_1)$-DP (consider $r+i\dist(\cdot, C)$ as the regularizer) by Lemma \ref{lem:part1_simple} and Part 1 above, this would be sufficient to prove that the release of $\phi(W)$ is $(\varepsilon_1, \delta_1)$-DP by the Successive Approximation Lemma \ref{lem:sucessive}. 

    We fix a $W\in E_\alpha$, and denote for convenience by $\phi^i\equiv \phi^i(W)$ and $\phi\equiv \phi(W)$. We observe that $\phi\in C$, and if $\phi^i\in C$, then $\phi^i$ minimizes $\Jp(~\cdot~; \D)$ over $C$, and hence $\phi^i=\phi$. As such, to show convergence of $\phi^i$ towards $\phi$, it is sufficient to show that there exists an index $N\ge 0$ such that, for all $i\ge N$, $\phi^i\in C$. 

    \red{If $\phi^0\in C$, then the result is already shown. As such, assume $\phi^0\not\in C$.}

    Consider the set of points $\theta\in \R^d$ such that $\Jp(\theta; \D)\le \Jp(\phi; \D)$, and denote this set by $H$. Then, by $\alpha\sigma^2$-strong convexity of $\Jp(~\cdot~;\D)$, it must hold that, if we define $D=\sqrt{\frac{2}{\alpha\sigma^2}[\Jp(\phi; \D)-\Jp(\phi^0; \D)]}\red{>0}$, for all $\theta\in H$,
    \[
        \frac{\alpha\sigma^2}{2}\|\theta-\phi^0\|^2\le \Jp(\theta; \D)-\Jp(\phi^0; \D)\le \Jp(\phi; \D)-\Jp(\phi^0; \D)=\frac{\alpha\sigma^2}{2}\cdot D^2.
    \]
    Moreover, since $\phi\in C$ is clearly also in $H$, it must hold that, for any point $\theta\in H$, $\dist(\theta, C)\le \|\theta-\phi\|\le 2D$. Specifically, over $H$, the quantity $\lambda_\theta=\frac{2D}{\dist(\theta, C)}\ge 1$. Now we note that $\Jp$ is convex, and hence, for $x=\theta$, $y=\proj_C(\theta)-\theta$ and $\lambda=\lambda_\theta$, Lemma \ref{S12::lem:weird_convex} gives that 
    \tight{\begin{align*}
        \Jp(\proj_C(\theta); \D)-\Jp(\theta; \D)
        &\le \frac{\Jp(\theta+2D\cdot \frac{\proj_C(\theta)-\theta}{\dist(\theta, C)}; \D)-\Jp(\theta; \D)}{\frac{2D}{\dist(\theta, C)}}\\
        &\le \frac{M}{2D}\cdot \dist(\theta, C),
    \end{align*}}
    {\[
        \Jp(\proj_C(\theta); \D)-\Jp(\theta; \D)\le \frac{\Jp(\theta+2D\cdot \frac{\proj_C(\theta)-\theta}{\dist(\theta, C)}; \D)-\Jp(\theta; \D)}{\frac{2D}{\dist(\theta, C)}}\le \frac{M}{2D}\cdot \dist(\theta, C),
    \]}
    where $M$ is defined as the supremum of the function $(\theta, v)\mapsto \Jp(\theta + 2D\cdot v; \D)-\Jp(\theta; \D)$ over $\bar H\times \partial B(0, 1)$ (exists as the function is continuous over a compact). As such, we get that, for all $\theta\in H$ such that $\theta\not\in C$, for $i>\frac{M}{2D}$,
    \begin{align*}
        \Jpi(\theta; \D)
        &=\Jp(\theta; \D)+i\dist(\theta, C) \\
        &=\Jp(\proj_C(\theta); \D)+i\dist(\theta, C)-\left[\Jp(\proj_C(\theta); \D)-\Jp(\theta; \D)\right] \\
        &\ge \Jp(\proj_C(\theta); \D)+i\dist(\theta, C)-\frac{M}{2D}\cdot \dist(\theta, C) \\
        &> \Jp(\proj_C(\theta); \D) \\
        &\ge \Jp(\phi; \D).
    \end{align*}
    We note that
    \[
        \Jp(\phi^i; \D)\le \Jpi(\phi^i; \D)\le \Jpi(\phi; \D)=\Jp(\phi; \D),
    \]
    and hence $\phi^i\in H$ for all $i$. As such, the above strict inequality cannot hold for $\theta=\phi^i$, and we conclude that $\phi^i\in C$ for $i>\frac{M}{2D}$, as wanted.
\end{proof}

\subsection{Release of Noisy Error}

\begin{lemma}[Conditional Release of Noisy Error]\label{lem:part2}
    \red{Assume that the Problem Assumption \ref{ass:problem} and the Random Matrix Assumption \ref{ass:distribution} hold.
    The release of $\theta_\final-\theta_\priv$, conditioned on $E_{\alpha}$ and the realized value of $\theta_\priv$, is $(\varepsilon_2, \delta_2)$-DP, uniformly in the realization $W\in E_{\alpha}$, for 
    \[
        \tilde \sigma\ge \frac{2}{\varepsilon_2}\sqrt{\frac{2\tau}{\alpha\sigma^2}}\sqrt{2\ln(1.25/\delta_2)},
    \]
    where $\alpha$ is given by the Random Matrix Assumption \ref{ass:distribution}.}
\end{lemma}
\red{\begin{proof}
    Condition on $E_\alpha$, the preceding releases, and $\theta_\priv=t$. For every dataset $\D$ and every conditional realization $W\in E_\alpha$, the optimization error $e_\D(W)=\theta_\approx-t$ satisfies, since $\Jp(\cdot, \D)$ is $\alpha\sigma^2$-strongly convex under $E_\alpha$,
    \[
        \frac{\alpha\sigma^2}{2}\|e_\D(W)\|^2=\frac{\alpha\sigma^2}{2}\|\theta_\approx-\theta_\priv\|^2\le \Jp(\theta_\approx; \D)-\Jp(\theta_\priv; \D) \le \tau.
    \]
    Therefore, for any adjacent datasets $\D\sim \D'$ and arbitrary conditional realizations $W, W'\in E_\alpha$, it holds that 
    \[
        \|e_\D(W)-e_{\D'}(W')\|\le 2\sqrt{\frac{2\tau}{\alpha\sigma^2}}.
    \]
    The Gaussian Mechanism \ref{lem:gauss} establishes, for every conditional realization $W, W'\in E_\alpha$, $(\varepsilon_2, \delta_2)$-DP. Integrating over the conditional distributions yields the wanted result.
\end{proof}
}
\section{Utility Analysis - Proof of Theorem \ref{thm:utility}}\label{sec:proof_utility}
\red{
Consider the following decomposition
\begin{subequations}
\begin{align}
    \Exp{\J (\theta_\final; \D)-\J(\theta_\exact; \D)}
    =& ~~~~\Exp{\J(\theta_\final; \D)-\J(\theta_\approx; \D)} \label{eq:utility:a}\\
    &+ \Exp{\J(\theta_\approx; \D)-\Jp(\theta_\approx; \D)} \label{eq:utility:b}\\
    &+ \Exp{\Jp(\theta_\approx; \D)-\Jp(\theta_\priv; \D)} \label{eq:utility:c}\\
    &+ \Exp{\Jp(\theta_\priv; \D)-\Jp(\theta_\exact; \D)} \label{eq:utility:d}\\
    &+ \Exp{\Jp(\theta_\exact; \D)-\J(\theta_\exact; \D)} \label{eq:utility:e}.
\end{align}
\end{subequations}
We first note that \eqref{eq:utility:d} is nonpositive as $\theta_\priv$ is the minimizer of $\Jp(\cdot, \D)$. Second, the quantity \eqref{eq:utility:b} is bounded by
\begin{align*}
    \Exp{\J(\theta_\approx; \D)-\Jp(\theta_\approx; \D)}
    &= \Exp{-\frac{\sigma^2}{2}(\theta_\approx-\theta_c)^TW(\theta_\approx-\theta_c)-\langle a, \theta_\approx-\theta_c\rangle} \\
    &= -\frac{\sigma^2}{2}\Exp{(\theta_\approx-\theta_c+\sigma^{-2}W^{-1}a)^TW(\theta_\approx-\theta_c+\sigma^{-2}W^{-1}a)}+\frac{1}{2\sigma^2}\Exp{a^TW^{-1}a} \\
    &\le \frac{1}{2\sigma^2}\Exp{a^TW^{-1}a} \\
    &\le \frac{\sigma_0^2}{2\sigma^2}\Exp{\tr(W^{-1})} \\
    &\le \frac{\sigma_0^2d\nu }{2\sigma^2}.
\end{align*}
The quantity \eqref{eq:utility:c} is bounded by $\tau$ by assumption. We may bound \eqref{eq:utility:e} by using
\begin{align*}
    \Exp{\Jp(\theta_\exact; \D)-\J(\theta_\exact; \D)} 
    &= \Exp{\frac{\sigma^2}{2}(\theta_\exact-\theta_c)^TW(\theta_\exact-\theta_c)}+\Exp{\langle a, \theta_\exact-\theta_c\rangle}\\
    &\le \frac{\sigma^2}{2}\cdot (\theta_\exact-\theta_c)^T\Exp{W}(\theta_\exact-\theta_c) \\
    &\le \frac{\mu\sigma^2}{2}\cdot \|\theta_\exact-\theta_c\|^2.
\end{align*}
Finally, by $L$-smoothness of $\ell(\cdot; z)$ and since $b$ has zero mean and is independent of $\theta_\approx$, it holds that
\tight{\begin{align*}
    \Exp{\ell(\theta_\final;z)}
    &=\Exp{\ell(\theta_\approx+b;z)}\\
    &\le \Exp{\ell(\theta_\approx;z)+\langle \nabla \ell(\theta_\approx; z), b\rangle+\frac{L}{2}\|b\|^2}\\
    &= \Exp{\ell(\theta_\approx;z)+\frac{L}{2}\|b\|^2},
\end{align*}}
{\begin{align*}
    \Exp{\ell(\theta_\final;z)}
    =\Exp{\ell(\theta_\approx+b;z)}
    \le \Exp{\ell(\theta_\approx;z)+\langle \nabla \ell(\theta_\approx; z), b\rangle+\frac{L}{2}\|b\|^2}
    = \Exp{\ell(\theta_\approx;z)+\frac{L}{2}\|b\|^2},
\end{align*}}
and by the bounded subgradients assumption on $r$, we get that, for any $g\in \partial r(\theta_\final)$,
\[
    \Exp{r(\theta_\final)}=\Exp{r(\theta_\approx+b)}\le \Exp{r(\theta_\approx)+\langle g, b\rangle}\le \Exp{r(\theta_\approx)+G \|b\|}.
\]
Combining, we obtain a bound for \eqref{eq:utility:a}, namely
\[
    \Exp{\J(\theta_\final; \D)-\J(\theta_\approx; \D)}\le \frac{nL}{2} \Exp{\|b\|^2}+G \Exp{\|b\|}.
\]
As such, combining all the bounds, we obtain
\[
    \Exp{\J (\theta_\final; \D)-\J(\theta_\exact; \D)}\le \frac{nL}{2} \Exp{\|b\|^2}+G \Exp{\|b\|}+\tau+\frac{\mu\sigma^2}2\cdot \|\theta_\exact-\theta_c\|^2 + \frac{\sigma_0^2d\nu }{2\sigma^2}.
\]
The result follows by recalling that $b\sim \mathcal N(0, \tilde\sigma^2I)$, and hence $\Exp{\|b\|^2}=d\tilde\sigma^2$ and $\Exp{\|b\|}\le \sqrt{d}\tilde\sigma$.
}
\section{Experimental Details}

\subsection{Stopping Criterion}\label{ssec:stopping}

The stopping criterion in Equation \eqref{eq:approximate_def} is not directly verifiable in practice as it requires the optimal value of Problem \eqref{eq:p_priv}. We therefore introduce a computable surrogate condition, based on first-order optimality conditions. In fact, we consider the stopping condition 
\begin{equation}\label{eq:surrogate_stopping}
    \dist(0, \partial \Jp(\theta; \D)+N_C(\theta))\le \sqrt{2\tau\red{\lambda_{\min}(W)}\sigma^2},
\end{equation}
where $N_C$ is the normal cone of $C$. To see that Equation \eqref{eq:surrogate_stopping} implies the stopping condition in Equation \eqref{eq:approximate_def}, we note that $\theta\mapsto \Jp(\theta; \D)+\iota_C(\theta)$ is $\red{\lambda_{\min}(W)}\sigma^2$-strongly convex, such that 
\[
    \dist(0, \partial \Jp(\theta; \D)+N_C(\theta))^2\ge 2\red{\lambda_{\min}(W)}\sigma^2 \left[\Jp(\theta; \D)-\min_{\theta\in C}\Jp(\theta; \D)\right].
\]

\subsection{Stochastic Three-Operator Splitting}\label{sec:ass_compl}

While the proposed mechanism is conceptually simple, it requires solving a perturbed problem which is nontrivial due to its finite-sum structure, nonsmooth regularizer, and constraints. In this section, we show that the perturbed problem can be solved efficiently by leveraging its composite structure.

We consider a realization of \red{$(\tilde\theta, W)$}, and focus on how to (approximately) solve Problem \eqref{eq:p_priv}. Specifically, the problem will be fixed, and the randomness will arise through the algorithm alone.

A key observation is that Problem \eqref{eq:p_priv} may be written as the sum of a smooth finite-sum problem with two convex regularizers, namely
\begin{equation*}
    \argmin_{\theta\in \R^d}\left\{\sum_{i=1}^n \tilde \ell(\theta; z_i)+r(\theta)+\iota_C(\theta)\right\},
\end{equation*}
where $\tilde \ell(\theta;z)=\ell(\theta; z)+\frac{\sigma^2}{2n}(\theta-\tilde\theta)^TW(\theta-\tilde\theta)$ is smooth and convex, and $\iota_C$ is the indicator of the convex set $C$. We note this is not specific to our resulting problems, and would also be the structure of Problem \eqref{eq:p_priv_lin}, with $\tilde \ell(\theta; z)=\ell(\theta; z)+\frac{a^T\theta}{n}+\frac{\Delta}{2n}\|\theta\|^2$.

This structure suggests the use of the \textit{Three-Operator Splitting} scheme \citep{davis_threeoperator_2017a}. The scheme was originally introduced in the deterministic setting, treating the finite-sum as a single operator, and has since been extended in the stochastic setting, in which a random term of the sum is used at each iteration \citep{yurtsever_stochastic_2016,yurtsever_three_2021,cortild_stochastic_2026}. The resulting algorithm is described in Algorithm \ref{alg:StoTOS}.

\begin{algorithm}[H]
\renewcommand{\thealgorithm}{StoTOS}
\caption{(Stochastic Three-Operator Splitting)}\label{alg:StoTOS}
\begin{algorithmic}
\Require a step-size $\gamma$, a relaxation sequence $(\lambda_k)$, and an initial guess $x_0\in \R^d$.
\For{$k=0, \ldots, K$}
    \State Draw $i_k\in \{1,  \ldots, n\}$ randomly and independently.
    \State Set $\theta_k=\proj_C(x_k)$.
    \State Set $z_k=\prox_{\gamma r}(2\theta_k-x_k-\gamma n\nabla \tilde \ell(\theta_k; z_{i_k}))$.
    \State Set $x_{k+1}=x_k-\lambda_k\theta_k+\lambda_kz_k$.
    \State Update $k\to k+1$.
\EndFor \\
\Return $\theta_K$
\end{algorithmic}
\end{algorithm}

The work \cite{yurtsever_stochastic_2016} introduced stochasticity into the scheme for strongly convex objectives, whereas \cite{yurtsever_three_2021} studied non-strongly convex objectives. Both works make a strong assumption resembling uniformly bounded variance of the stochastic oracle, which is not verified in our setting without additional assumptions. In contrast, \cite{cortild_stochastic_2026} rely on a weaker assumption, by only assuming a finite variance of the stochastic oracle at a solution of the problem. 

Convergence to a solution of the perturbed problem can be established in several important cases. In particular, when the constraint set $C$ is an affine subspace (including the unconstrained setting), the method provably converges with accompanying rate statements. More broadly, more general conditions are considered in \cite{cortild_stochastic_2026}. 

We also note that when the problem is unconstrained, the function $\iota_C$ is identically zero, and can therefore be disregarded. In this case, the Three-Operator Splitting algorithm reduces to the Backward-Forward algorithm \citep{attouch_backward_2018}, and Algorithm \ref{alg:StoTOS} would reduce to a stochastic variant of it, in which the projection step is trivial. 

\subsection{Experimental Setup}\label{sec:experimental}

The goal of the experiment in Section \ref{ssec:LOP_numerical} is to isolate the dependence of the mechanisms on the constraint diameter $\kappa$, and thereby the dependence on the bounded gradients assumption required by LOP.

\paragraph{Data Generation.} \red{To generate the data, we generate $n$ independent samples $x_i\sim \mathcal N(0, I)$, and rescale them to lie within $[-\xi, \xi]^d$. We then generate a random $\theta_*\sim \mathcal N(0, I)$ and a random $c_i\sim \mathcal N(0, 1)$, and define $y_i=\text{Clip}(x_i^T\theta_*+c_i, [-1, 1])$. We consider $\xi=5$.}

\paragraph{Problem Setup.} We consider the regularization parameter $\omega=1$, and constraint sets $C=[-\kappa, \kappa]^d$ for varying $\kappa>0$. 

\textbf{Algorithmic Setup.} We solve the resulting problem through Algorithm \ref{alg:StoTOS}. For this purpose, we run $1000$ iterations using the relaxation parameters \red{$\lambda_k=(k+1)^{-1/2}$} and the step-size $\gamma=1/\operatorname{Lip}(n\nabla \tilde \ell)$. 

\textbf{Mechanism Setup.} \red{The experiment isolates the objective-perturbation component and is not intended to simulate the complete inexact mechanism. When running either mechanism, we therefore omit the final Gaussian noise term $b$. We perform the following split: $(\varepsilon_0, \varepsilon_1, \varepsilon_2)=(\varepsilon/2, \varepsilon/2, 0)$ and $(\delta_0,\delta_1, \delta_2, \delta_3)=(\delta/3, \delta/3, 0, \delta/3)$. We select the Wishart hidden dimension to be $m=2d$. For LOP-Clip, we clip each gradient to have norm at most the prescribed value.}

\textbf{Evaluation Protocol.} \red{In all cases}, we evaluate the empirical risk. We run the algorithm $10$ times and report the average. 

\textbf{Computer Resources.} The experiment is written in Python 3.13, and was executed on an Apple Silicon MacBook Pro, with a M5 chip and 16GB of RAM.

\red{\textbf{Code Availability.} The code is available on the authors GitHub page: \url{https://github.com/DanielCortild/QOP-Mechanism}.}

\subsection{Statistical Significance Reporting}\label{ssec:stats}

In the interest of transparency, we report the average and maximum standard deviations (Std) and standard errors (SE), and the average runtime obtained in the process of obtaining Figure \ref{fig:risk}.

\begin{table}[H]
    \centering
    \caption{Summary statistics for Figure \ref{fig:risk} with $(n, d, \varepsilon, \delta)=(300, 70, 0.5, 0.01)$.}
    \begin{tabular}{|l|l|l|l|l|l|}\hline
        & \textbf{Avg Std} & \textbf{Max Std} & \textbf{Avg SE} & \textbf{Max SE} & \textbf{Avg Runtime (in s)} \\ \hline
        \textbf{LOP} & 47229.6 & 1.07102e+06 & 14935.3 & 338686 & 0.00821865 \\
        \textbf{LOP-Clip-1e4} & 95.691 & 360.477 & 30.2602 & 113.993 & 0.00900318 \\
        \textbf{LOP-Clip-1e5} & 4623.81 & 35825.8 & 1462.18 & 11329.1 & 0.00900562 \\
        \textbf{QOP} & 7.12276e-05 & 0.000127098 & 2.25241e-05 & 4.0192e-05 & 0.0114504 \\\hline
    \end{tabular}
\end{table}

\begin{table}[H]
    \centering
    \caption{Summary statistics for Figure \ref{fig:risk}  with $(n, d, \varepsilon, \delta)=(300, 70, 0.01, 0.00001)$.}
    \begin{tabular}{|l|l|l|l|l|l|}\hline
        & \textbf{Avg Std} & \textbf{Max Std} & \textbf{Avg SE} & \textbf{Max SE} & \textbf{Avg Runtime (in s)} \\ \hline
        \textbf{LOP} & 44411.3 & 1.15901e+06 & 14044.1 & 366511 & 0.00823114 \\
        \textbf{LOP-Clip-1e4} & 575.211 & 2902.89 & 181.898 & 917.975 & 0.00897382 \\
        \textbf{LOP-Clip-1e5} & 16805.4 & 226260 & 5314.34 & 71549.7 & 0.00897671 \\
        \textbf{QOP} & 8.13505e-05 & 0.000156076 & 2.57253e-05 & 4.93557e-05 & 0.0115443 \\\hline
    \end{tabular}
\end{table}

\begin{table}[H]
    \centering
    \caption{Summary statistics for Figure \ref{fig:risk} with $(n, d, \varepsilon, \delta)=(400, 100, 0.5, 0.01)$.}
    \begin{tabular}{|l|l|l|l|l|l|}\hline
        & \textbf{Avg Std} & \textbf{Max Std} & \textbf{Avg SE} & \textbf{Max SE} & \textbf{Avg Runtime (in s)} \\ \hline
        \textbf{LOP} & 35137.4 & 834103 & 11111.4 & 263767 & 0.00833787 \\
        \textbf{LOP-Clip-1e4} & 31.5804 & 136.781 & 9.98659 & 43.2541 & 0.00917076 \\
        \textbf{LOP-Clip-1e5} & 1709.59 & 10474.9 & 540.618 & 3312.46 & 0.00909861 \\
        \textbf{QOP} & 3.48617e-05 & 5.62145e-05 & 1.10242e-05 & 1.77766e-05 & 0.0125463 \\\hline
    \end{tabular}
\end{table}

\begin{table}[H]
    \centering
    \caption{Summary statistics for Figure \ref{fig:risk} with $(n, d, \varepsilon, \delta)=(400, 100, 0.01, 0.00001)$.}
    \begin{tabular}{|l|l|l|l|l|l|}\hline
        & \textbf{Avg Std} & \textbf{Max Std} & \textbf{Avg SE} & \textbf{Max SE} & \textbf{Avg Runtime (in s)} \\ \hline
        \textbf{LOP} & 35451.7 & 710939 & 11210.8 & 224819 & 0.0081972 \\
        \textbf{LOP-Clip-1e4} & 225.999 & 1114.31 & 71.4671 & 352.374 & 0.00902137 \\
        \textbf{LOP-Clip-1e5} & 9831.14 & 108552 & 3108.88 & 34327 & 0.00900035 \\
        \textbf{QOP} & 4.30336e-05 & 8.13922e-05 & 1.36084e-05 & 2.57385e-05 & 0.0123227 \\\hline
    \end{tabular}
\end{table}

\end{document}